\documentclass{article}

\usepackage[preprint]{neurips_2026}

\usepackage[utf8]{inputenc} 
\usepackage[T1]{fontenc}    
\usepackage{hyperref}       
\usepackage{url}            
\usepackage{amsmath,amssymb}
\usepackage{booktabs}       
\usepackage{amsfonts}       
\usepackage{nicefrac}       
\usepackage{microtype}      
\usepackage{xcolor}         
\usepackage{graphicx,verbatim}
\usepackage{algorithm}
\usepackage{enumitem}
\usepackage{algorithmic}
\usepackage{bbm}
\usepackage{booktabs}
\usepackage{multirow}
\usepackage{amsthm}
\usepackage{booktabs}
\usepackage{wrapfig}

\newtheorem{theorem}{Theorem}[section]
\newtheorem{proposition}[theorem]{Proposition}
\newtheorem{lemma}[theorem]{Lemma}

\newtheorem{assumption}[theorem]{Assumption}

\title{Topology-Driven Transferability Estimation for 3D Medical Vision Foundation Models}

\author{%
  Jiaqi Tang$^{1,3,6}$\thanks{These authors contributed equally to this work.}
  \and
  Shaoyang Zhang$^{1}$\footnotemark[1]
  \and
  Fandong Zhang$^{4}$\footnotemark[1]
  \and
  Shu Zhang$^{4}$
  \and
  Yang Liu$^{7}$
  \and
  Qingchao Chen$^{1,3,6}$\thanks{Corresponding author: \texttt{qingchao.chen@pku.edu.cn}} \\
  $^{1}$National Institute of Health Data Science, Peking University, Beijing, China \\
  $^{3}$Institute of Medical Technology, Peking University, Beijing, China \\
  $^{4}$Deepwise Co., Ltd. \\
  $^{6}$State Key Laboratory of General Artificial Intelligence, Peking University, Beijing, China \\
  $^{7}$Wangxuan Institute of Computer Technology, Peking University, Beijing, China
}

\begin{document}

\maketitle

\begin{abstract}
The growing number of medical vision foundation models highlights the need for effective model selection. However, mainstream selection methods rely on exhaustive fine-tuning, which is computationally expensive. Most of the existing Transferability Estimation (TE) metrics are primarily designed for image-level classification. Their reliance on parametric assumptions or global statistical aggregations fails to preserve spatial relationships and fine-grained boundary details, which are crucial for the segmentation task. Additionally, while image-level tasks typically process a single feature vector per input, dense prediction tasks in 3D medical imaging require voxel-wise evaluation against dense annotations. To bridge these gaps, we propose a \textit{non-parametric, topology-driven} framework that estimates transferability directly from the alignment between the sparse 1-skeleton graph of dense features and semantic labels via Minimum Spanning Trees (MST). We decouple the alignment into two complementary geometric scales: Local Boundary-Aware Topological Consistency (LBTC) to assess boundary separability, where we rigorously prove that the MST leakage rate serves as a finite-sample lower bound on the Bayes error; and Global Representation Topology Divergence (GRTD) to evaluate the overall anatomical layout. Crucially, we formally justify a counterintuitive mechanism: Although without fine-tuning, the randomly initialized segmentation decoder acts as a topology-preserving spatial projector, reducing the variance of pairwise distance estimates and stabilizing global alignment evaluation. Fused via a task-adaptive gating mechanism, these dual metrics adapt to diverse clinical complexities. Evaluated on a large-scale benchmark of 114,000 3D medical volumes across diverse anatomical tasks, our topological framework achieves state-of-the-art transferability estimation with an average weighted Kendall's $\tau$ (outperforming the best baseline by 0.36) while accelerating evaluation by 56\(\times\).
\end{abstract}

\section{Introduction}

The rise of self-supervised learning (SSL) has produced a growing zoo of 3D medical vision foundation models~\cite{chen2020simple,zhou2021models,He2022MAE,chen2023masked,wu2024voco,Tang2022SwinUNETR}. However, as shown in Figure~\ref{fig:teaser} (Left), under the evaluation setup of fine-tuning in the model zoo, no single encoder universally dominates across diverse downstream segmentation tasks that span various anatomical regions (e.g., brain, abdomen, kidney, head and neck) and imaging modalities (MRI and CT)~\cite{wald2025openmind}. This diversity creates a simple but computationally brutal question: given a new segmentation task and several pre-trained encoders, which one should be selected for fine-tuning? Since exhaustive fine-tuning over the full model zoo is prohibitively expensive~\cite{wald2025openmind}, there is \textit{a pressing need for a training-free Transferability Estimation (TE) method capable of ranking models based solely on their features and the segmentation annotations.}

\begin{wrapfigure}{r}{0.5\textwidth}
    \centering
    \includegraphics[width=0.5\textwidth]{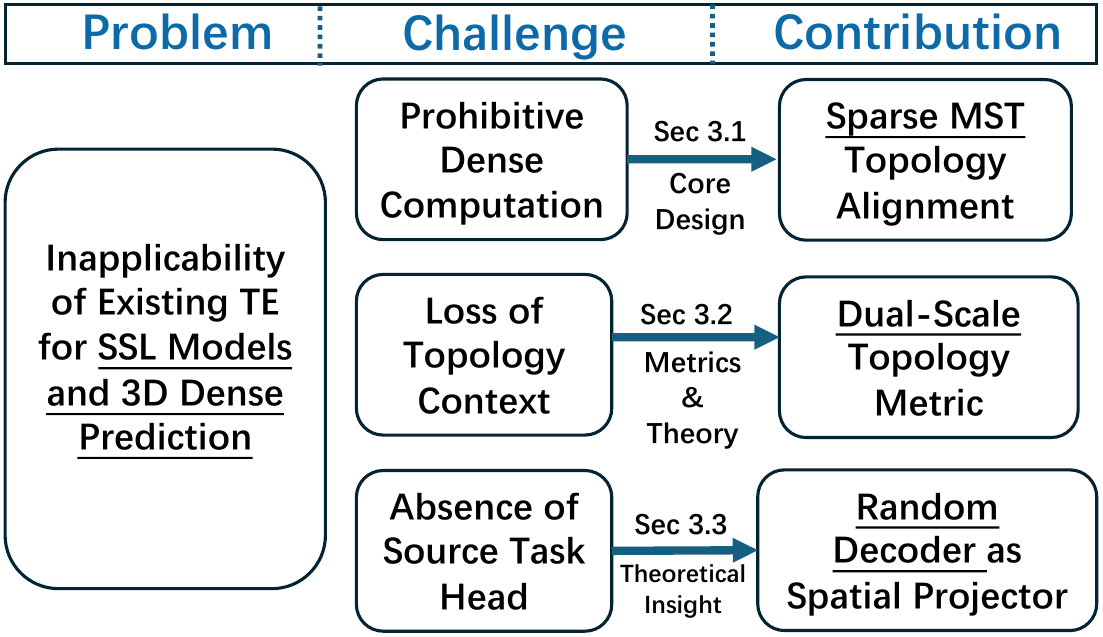}
    \vspace{-0.6cm}
    \caption{Motivation and core contributions of our topology-driven TE framework.}
    \vspace{-0.3cm}
    \label{fig:mindmap}
\end{wrapfigure}

Two main lines of existing transferability metrics are primarily designed for image classification tasks: (1) label-comparison metrics that evaluate the alignment between source and target label distributions~\cite{tran2019transferability,nguyen2020leep}, and (2) embedding-based metrics that assess the compatibility between pre-trained features and target labels~\cite{you2021logme,pandy2022transferability,bao2019information,huang2022frustratingly}.However, applying these metrics to SSL-based 3D medical foundation models is fundamentally problematic, as illustrated in Figure~\ref{fig:mindmap}.
First, \textbf{the absence of a source classification head in SSL models creates an architectural mismatch.} Previous label comparison-based methods, such as LEEP~\cite{nguyen2020leep} rely on a pre-trained classification head to compute. However, most medical foundation models are trained under SSL methods, providing no such task head that is consistent with the downstream task.
At the same time, \textbf{embedding-based metrics rely on global statistical assumptions that fail to capture the fine-grained spatial dependencies crucial for 3D segmentation.} Metrics like LogME~\cite{you2021logme}, GBC~\cite{pandy2022transferability} typically evaluate feature-label alignment by fitting simple parametric models (e.g., Gaussians). Designed for image-level classification, these methods treat features as independent samples or aggregate them across the entire image. When applied to 3D medical volumes, this global aggregation discards the spatial relationships between voxels, ignoring the anatomical structure and boundary details that are essential for the segmentation task, as shown in Figure~\ref{fig:teaser}(right).
Furthermore, \textbf{applying these metrics to 3D segmentation is computationally prohibitive.} A single medical volume contains millions of voxels with highly imbalanced classes. Computing pairwise feature-label similarities for such massive amounts of data requires excessive time and memory, making it impractical for evaluating a large pool of candidate models.
These observations motivate a paradigm shift specifically tailored to SSL-based 3D vision foundation models for dense prediction. 
Instead of relying on any architectural consistency or parametric assumptions, we propose a \textit{non-parametric, topology-driven framework that can be computed directly from an SSL-based encoder paired with a randomly initialized segmentation decoder}, without any reliance on a source task head.

\begin{figure*}[t]

    \includegraphics[width=0.32\linewidth, height=0.37\linewidth]{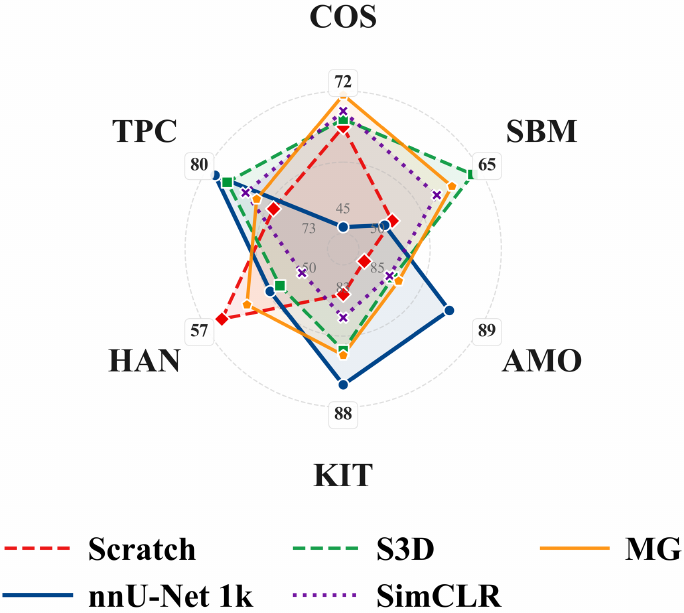} 
    \includegraphics[width=0.67\linewidth, height=0.37\linewidth]{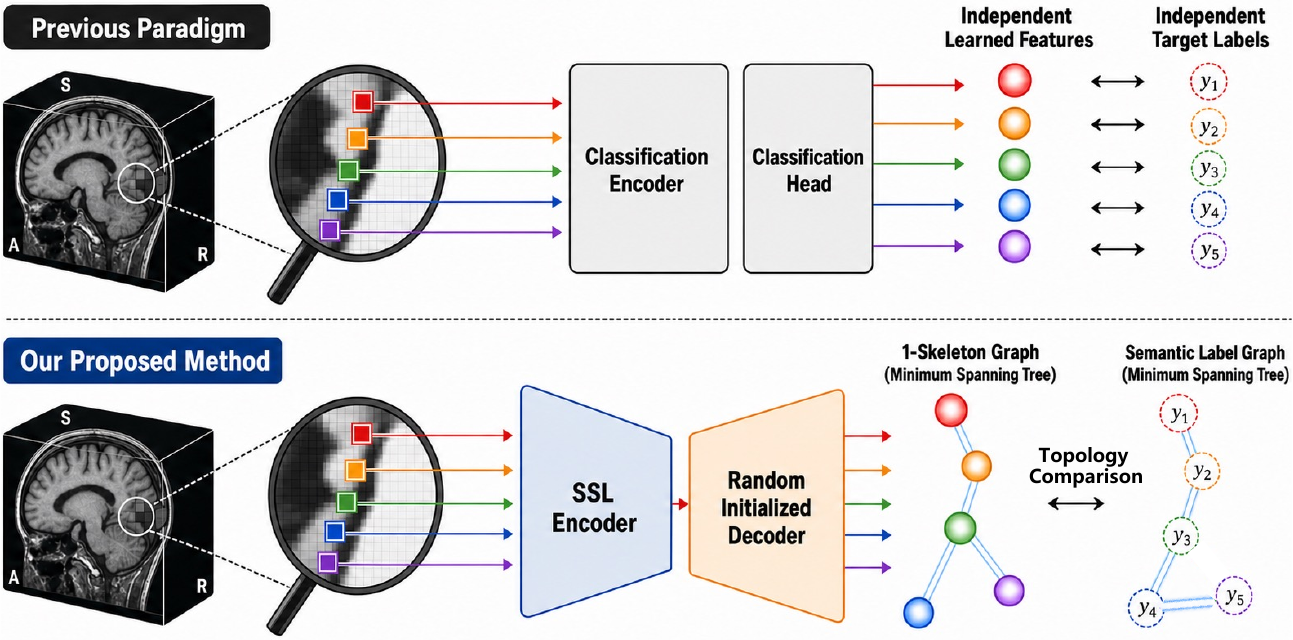} 

    \caption{ \textbf{Left:} The radar chart evaluates models on datasets with diverse anatomical regions and modalities: including COS~\cite{chen2022carotid} (carotid vessel, MRI), SBM~\cite{grovik2020deep} (brain, MRI), AMO~\cite{ji2022amos} (abdominal organs, CT/MRI), KIT~\cite{heller2021state} (kidney, CT), HAN~\cite{podobnik2023han} (head and neck, MRI), and TPC~\cite{yang2024topcow} (brain vessels, MRI). Fine-tuning performance of diverse foundation models from~\cite{wald2025openmind} varies across downstream datasets. \textbf{Right:} Previous methods treat dense voxels as independent samples and loses essential spatial context. Our topology-driven approach captures the structural geometry of the feature space. By constructing a Minimum Spanning Tree (MST), we estimate transferability by aligning the feature topology with the semantic label graph. }
    \label{fig:teaser}
    \vspace{-4mm}
\end{figure*}

To be specific, as shown in Figure~\ref{fig:teaser} (right-bottom), we leverage the Minimum Spanning Tree (MST) to extract a sparse 1-skeleton graph structure from the dense feature space and compare its alignment with corresponding semantic labels MST, bypassing statistical assumptions while reducing the computational overhead of dense voxel evaluation. Furthermore, recognizing that medical segmentation hinges on both local boundary precision and global anatomical layouts, we decouple our MST-based evaluation into two complementary geometric scales to compare with the semantic label.
\textbf{For local boundary separability}, we compute Local Boundary-Aware Topological Consistency (LBTC) by zooming into the transition zones between adjacent tissues. We calculate the Topological Leakage Rate as the fraction of MST edges that cross different semantic classes, where we rigorously prove that this leakage rate serves as a finite-sample lower bound on the Bayes error. \textbf{For global structural layout}, we propose the Global Representation Topology Divergence (GRTD) metric to evaluate how well features from different organs are separated into distinct clusters on a macro scale. Finally, \textit{fused via a task-complexity-driven gating function}, these metrics adapt from focal lesions to multi-organ layouts, avoiding dense-regime computational bottlenecks while providing a comprehensive, training-free model ranking.
Crucially, we reveal a counter-intuitive mechanism that is essential for our global metric (GRTD): although without finetuning, passing embeddings through a randomly initialized decoder acts as a topology-preserving spatial projector. We formally justify how this random projection reduces variance while preserving the underlying manifold through a Johnson--Lindenstrauss-style argument~\cite{matouvsek2008variants}.

We validate our framework on a benchmark of SSL foundation models pre-trained on 114{,}000 3D medical volumes and 7 downstream segmentation tasks spanning four anatomical regions and two modalities, under both in-distribution(ID) and cross-region/cross-modality out-of-distribution (OOD) scenarios. Beyond achieving state-of-the-art performance by  a margin of 0.36 in weighted Kendall's $\tau$, our work is, to the best of our knowledge, the first to pioneer transferability estimation for dense segmentation in vision foundation models. 
Our method yields three key advantages: (i) it captures fine-grained spatial topologies via the local branch (LBTC), theoretically proven as a finite-sample lower bound on the optimal Bayes error; (ii) it assesses macro-anatomical layouts alignment through the global branch (GRTD), utilizing a randomly initialized decoder as a spatial projector to stabilize global TE; and (iii) it enables highly efficient computation via sparse minimum spanning trees over feature embeddings.


\section{Related Work}

\subsection{3D Medical Foundation Models}

The rise of large-scale self-supervised learning has yielded a rapidly expanding zoo of 3D medical foundation models, each forged under a distinct pre-training objective that leaves a characteristic geometric imprint on its representations. Reconstruction-based paradigms such as MAE\citep{He2022MAE}, SimMIM \citep{chen2023masked}, GreenMIM \citep{huang2022green}, CXR-MAE \citep{xiao2023delving}, and Models Genesis \citep{zhou2021models} train encoders via masked patch prediction, cultivating fine-grained sensitivity to local anatomical textures.
Contrastive methods, including SimCLR \citep{chen2020simple}, VoCo \citep{wu2024voco}, LVM-Med \citep{mh2023lvm}, learn discriminative geometry through instance- or graph-level positive-negative discrimination.
Self-distillation paradigms including RAD-DINO \citep{perez2025exploring}, DINOv3
\citep{simeoni2025dinov3}, and Phikon-v2 \citep{filiot2024phikon} distill semantic structure without explicit negatives, implicitly shaping class boundary representations.
The algorithmic diversity is further amplified by 3D volumetric architectures expressly designed for spatial complexity: SwinUNETR-v2 \citep{hatamizadeh2021swin}, VISTA3D
\citep{he2025vista3d}, Triad \citep{wang2025triad}, and OmniRad \citep{zedda2026omnirad}
push the boundaries of spatial reasoning and cross-center generalization.
Consequently, this diversity introduces a computationally brutal challenge. Since different downstream tasks favor distinct pre-training biases, as the model zoo expands, exhaustively fine-tuning each candidate for task-specific model selection has become computationally prohibitive, necessitating a highly efficient, training-free evaluation mechanism.

\subsection{Transferability Estimation}

Transferability Estimation (TE) tries to guide the model selection by predicting downstream performance using only pre-trained features, bypassing the prohibitive cost of full fine-tuning. Existing TE methods generally fall into two paradigms. 
\textbf{Label Comparison-Based Methods} \cite{tran2019transferability, nguyen2020leep} evaluate the alignment between source and target label distributions or pseudo-labels. However, they inherently rely on the source model's classifier and label space, rendering them fundamentally inapplicable to modern SSL foundation models that only provide feature extractors. 
\textbf{Source Embedding-Based Methods} bypass this limitation by quantifying the compatibility directly between pre-trained features and target labels. These approaches can be broadly categorized into: (i) \textit{Statistical \& Parametric Metrics}, such as LogME \cite{you2021logme}, $\mathcal{N}$LEEP \cite{li2021ranking}, H-score \cite{bao2019information}, PACTran \cite{ding2022pactran}, and GBC \cite{pandy2022transferability}, which estimate transferability by fitting linear models, Gaussian mixtures, or computing global class-conditional covariances; and (ii) \textit{Information-Theoretic \& Geometric Metrics}, such as PARC \cite{bolya2021scalable}, TransRate \cite{huang2022frustratingly}, OTCE \cite{tan2021otce}, ETran \cite{gholami2023etran}, and PEFTDiff \cite{khoba2025peftdiff}, which leverage pairwise similarities, optimal transport, or energy functions to measure feature-label alignment. Additionally, recent studies have begun exploring dataset transferability specifically within medical image classification \cite{juodelyte2024dataset}.
While effective for image-level classification, existing embedding-based metrics are ill-suited for dense 3D medical segmentation. Their reliance on global statistical aggregations or parametric assumptions failed to capture the spatial topologies and fine-grained boundary details essential for dense prediction.


 
\begin{figure*}[t]
    \includegraphics[width=0.99\linewidth, height=0.45\linewidth]{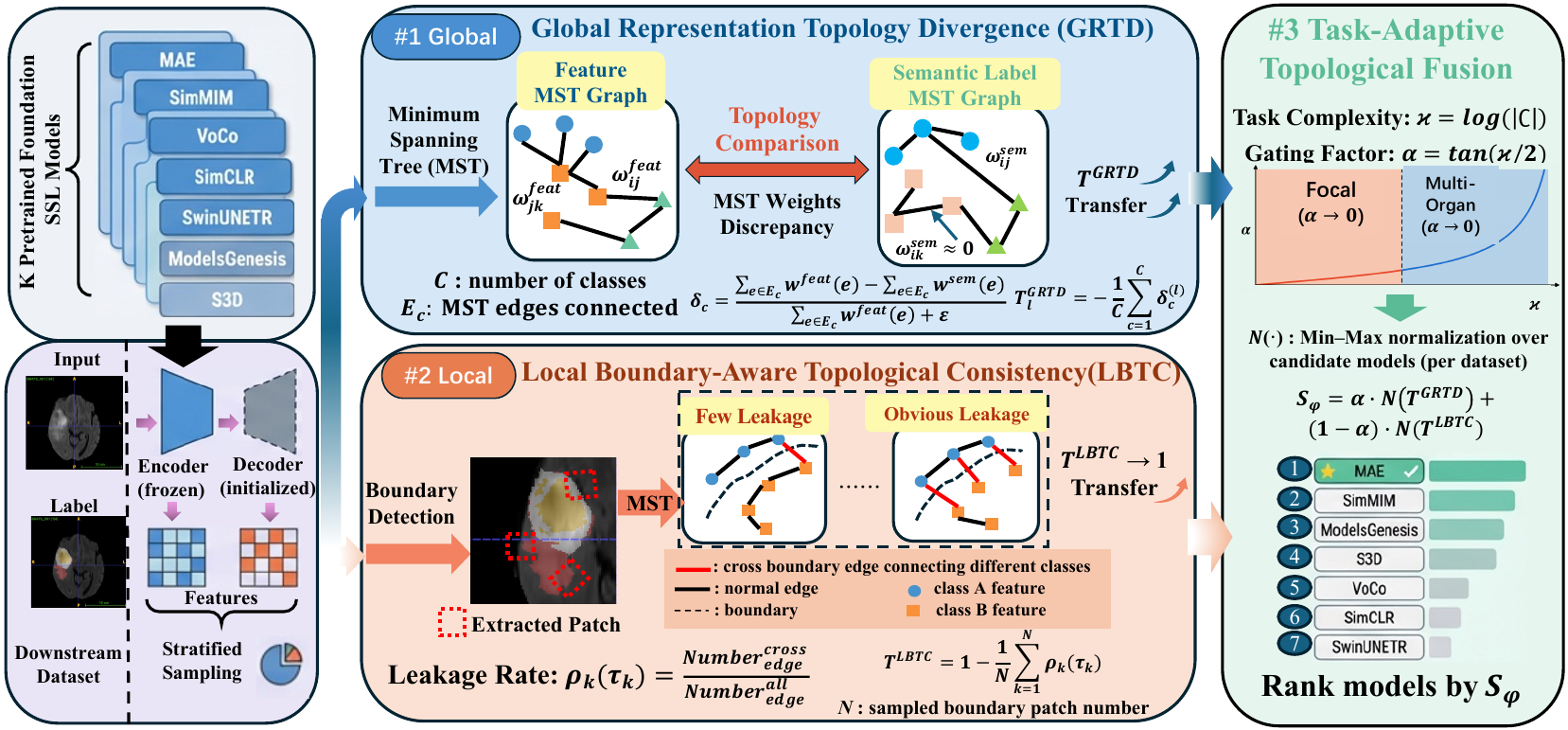}  
    \caption{Overview of our Topology-Driven TE framework. By constructing MST, the framework quantifies feature-label topological alignment across local boundary (LBTC) and global structural (GRTD) scales, fusing them via a task-adaptive gating mechanism to achieve accurate, training-free ranking of 3D medical foundation models.}
    \vspace{-0.3cm}
    \label{fig:framework}
\end{figure*}

\section{Methodology}
\label{sec:method}
To address the problem of selecting the optimal pre-trained encoder from a foundation model zoo \(\mathcal{M} = \{\phi_k\}_{k=1}^K\) for a downstream dense prediction task \(\mathcal{D}\), we aim to derive a training-free transferability score \(T(\phi, \mathcal{D})\). Existing Transferability Estimation (TE) metrics typically rely on statistical assumptions or require a pre-trained source classification head. Consequently, they fail to capture the complex spatial topologies and fine-grained boundary geometry inherent to 3D medical segmentation, and often suffer from prohibitive computational costs.

To overcome these limitations, we propose a non-parametric topological framework. As illustrated in Figure ~\ref{fig:framework}, consider an example of selecting a model for a brain tumor segmentation task. First, we pass the downstream MRI scans through a pre-trained model encoder and a randomly initialized decoder to extract dense spatial features. Next, as detailed in Section 3.1, we construct an MST over these features to extract their underlying topology without parametric assumptions. Building upon this, Section 3.2 formalizes our dual topological metrics to evaluate feature-label alignment: at the local scale, Local Boundary-Aware Topological Consistency (LBTC) assesses boundary separability by checking if local MST edges erroneously cross the tumor boundaries (i.e., leakage); at the macro scale, Global Representation Topology Divergence (GRTD) measures the overall structural layout by comparing the feature MST with the semantic label MST to verify if distinct tissues form well-separated clusters. These dual metrics are then fused via a task-adaptive gating mechanism based on the task's semantic complexity to output a final transferability score. Finally, Section 3.3 theoretically justifies our framework, demonstrating how spatial upsampling via an even randomly initialized decoder reduces the variance of pairwise distance estimates to robustly evaluate global alignment.

\subsection{MST-Based Topological Alignment}
\label{subsec:topology_proxy}

To overcome the limitation above, we first propose using the Minimum Spanning Tree (MST) to model the feature space. By linking data points through minimum-weight edges, the MST provides a discrete approximation of the feature space topology. This allows it to capture local neighborhood structures that are essential for recognizing boundaries while preserving the overall spatial layout of the features in a sparse, computationally efficient graph.

To further quantify transferability, we \textit{measure how well the topology of the pre-trained features aligns with the semantic structure required by the downstream task.} Given a set of \(N\) pixel-level feature--label pairs \(\mathcal{X} = \{(v_i, y_i)\}_{i=1}^N\), we construct two graphs on the same vertex set. The feature graph \(\mathcal{G}_{feat}\) uses Euclidean distances as edge weights, \(w_{ij}^{feat} = \|v_i - v_j\|_2\), reflecting the intrinsic geometry of the representation. Conversely, the semantic graph \(\mathcal{G}_{sem}\) adjusts these weights based on the task labels as \(w_{ij}^{sem} = \mathbb{1}[y_i \neq y_j] \min(\|v_i - v_j\|_2, \lambda)\),
where \(\lambda\) is an adaptive penalty set to the median pairwise Euclidean distance, which bounds the cost of crossing class boundaries. Let \(\mathcal{T}_{feat}\) and \(\mathcal{T}_{sem}\) denote the MSTs of \(\mathcal{G}_{feat}\) and \(\mathcal{G}_{sem}\), respectively. Here, \(\mathcal{T}_{feat}\) represents the natural clustering of the pre-trained features, while \(\mathcal{T}_{sem}\) represents an ideal structure where points within the same class are perfectly grouped. A pre-trained model tends to exhibit better transferability when its learned feature space aligns well with the target task's label distribution. Therefore, the structural divergence between \(\mathcal{T}_{feat}\) and \(\mathcal{T}_{sem}\) serves as a direct, training-free score of transferability.

In medical segmentation, a model must understand two distinct spatial properties: the overall layout of anatomies (e.g., separating the liver from the kidneys) and the fine-grained boundaries between adjacent tissues. As a single metric is hard to capture both, we decouple the structural comparison between \(\mathcal{T}_{feat}\) and \(\mathcal{T}_{sem}\) into two scales. Globally, we measure how well the features of different anatomies are separated into distinct clusters, which reflects the macro-anatomical layout. Locally, we zoom into the transition zones between adjacent tissues; measuring whether local feature neighborhoods cross semantic classes tells us how sharp and accurate the boundaries are.

\subsection{Dual Topological Requirements for Dense Prediction}
\label{subsec:dual_requirements}
To be specific, our framework consists of three key components.  First, we propose the Local Boundary-Aware Topological Consistency (LBTC) to assess the separability of features at local tissue boundaries. 
Second, we introduce the Global Representation Topology Divergence (GRTD) to measure how well the feature space captures the overall anatomical layout.
Finally, we combine these two metrics using a task-adaptive fusion mechanism, which dynamically adjusts their weights based on the semantic complexity of the target task.

\textbf{Local Boundary-Aware Topological Consistency (LBTC).} 
Boundaries are transition zones where adjacent tissues meet. For the segmentation task, a good representation should well separate the features of different classes in these local patches. We quantify this boundary sharpness using the Topological Leakage Rate, which measures the fraction of local Minimum Spanning Tree (MST) edges that cross between different classes in these ambiguous zones. Formally, for a local patch \(k\), let \(\mathcal{T}_k\) denote the MST constructed over its feature embeddings with vertex set \(\mathcal{V}_k\). The Topological Leakage Rate for this patch, denoted as \(\rho_k(\mathcal{T}_k)\), is defined as:
$
    \rho_k(\mathcal{T}_k) = \frac{1}{|\mathcal{V}_k| - 1} \sum_{(u,v) \in \mathcal{E}(\mathcal{T}_k)} \mathbb{I}(y_u \neq y_v)
$
where \(\mathcal{E}(\mathcal{T}_k)\) is the edge set of the MST, \(|\mathcal{V}_k| - 1\) represents the total number of edges in \(\mathcal{T}_k\), \(y_u\) and \(y_v\) are the ground-truth labels for nodes \(u\) and \(v\), and \(\mathbb{I}(\cdot)\) is the indicator function. A lower \(\rho_k\) indicates that the representation maintains sharp, well-separated boundaries within that local region.

\begin{assumption}[Feature--Label Alignment Principle]
\label{assum:transferability}
The downstream transferability of a pre-trained model is inversely related to the optimal Bayes error \(R^*\) of the features extracted by the fixed pre-trained encoder with respect to the target task labels~\cite{bao2019information, nguyen2020leep}.
\end{assumption}

Remarkably, our metric \(\rho\) provides exactly a theoretical guarantee to bound this Bayes error:

\begin{theorem}[MST Leakage Rate and Bayes Error]
\label{thm:leakage_bayes}
Assume the marginal distribution of $X$ is continuous with respect to the Lebesgue measure on $\mathbb{R}^d$. Let \(R_n^{\text{1NN}}\) be the empirical 1-Nearest Neighbor error and \(R^*\) the optimal Bayes error. The MST leakage rate \(\rho_n\) satisfies the finite-sample lower bound \(\rho_n \geq \frac{n}{2(n-1)}R_n^{\text{1NN}}\). As \(n \to \infty\),$\liminf_{n\to\infty} \rho_n \;\geq\; \frac{1}{2} R^{\text{1NN}} \;\geq\; \frac{R^*}{2} \quad \text{almost surely.}$

\end{theorem}

Theorem~\ref{thm:leakage_bayes} establishes that a vanishing combinatorial discrepancy (\(\rho_n \to 0\)) implies \(R^* \to 0\), meaning the representation is Bayes-separable. Therefore, we define a set of boundary anchors \(\partial\mathcal{Y}\) via the morphological gradient of the ground-truth masks. For each anchor \(c_k \in \partial\mathcal{Y}\), we extract a local patch \(\mathcal{P}_k\) and construct its local MST, \(\mathcal{T}_k\). The per-layer LBTC score aggregates these local leakage inconsistencies over \(N\) sampled boundary regions:
\begin{equation}
T^{\text{LBTC}}_l = 1 - \frac{1}{N} \sum_{k=1}^{N} \rho(\mathcal{T}_k^{(l)}).
\end{equation}
The final LBTC score \(T^{\text{LBTC}} \in [0, 1]\) is obtained by aggregating these per-layer scores across the encoder feature maps. A score approaching 1 implies strict topological separation (i.e., vanishing leakage and bounded Bayes error) even within fine-grained, high-frequency transition zones.

\textbf{Global Representation Topology Divergence (GRTD).} 
Beyond local boundaries, in a good feature space, features from different organs should simply group into distinct, separate clusters on a global scale.
We evaluate this macro-level separation using the Tree-Weight Discrepancy \(\Delta_W\) ~\cite{hero1999asymptotic} over the entire volume. Because \(\mathcal{T}_{sem}\) connects points within the same class at zero cost (as defined in Eq. 1), its total weight represents the minimum cost required to connect different organ clusters. Therefore, the discrepancy \(\Delta_W = W(\mathcal{T}_{feat}) - W(\mathcal{T}_{sem})\) captures the extra intra-class variance and measures how far apart these organ clusters actually are in the feature space.

To ensure our metric is robust against severe class imbalance and varying feature scales across different pre-trained models, we normalize this discrepancy for each class. Let \(E_c\) be the set of edges in \(\mathcal{T}_{feat}\) that connect to at least one node of class \(c\). We define the normalized per-class discrepancy \(\delta_c\) and the final layer-wise score \(T_l^{\text{GRTD}}\) as:
\begin{equation}
    \delta_c = \frac{\sum_{e \in E_c} w^{feat}(e) - \sum_{e \in E_c} w^{sem}(e)}{\sum_{e \in E_c} w^{feat}(e) + \varepsilon}, \quad T_l^{\text{GRTD}} = -\frac{1}{C} \sum_{c=1}^C \delta_c^{(l)}
\end{equation}
where \(\sum_{e \in E_c} w(e)\) calculates the total weight of all edges associated with class \(c\), and \(\varepsilon\) is a small constant to prevent division by zero. Since a smaller \(\delta_c\) indicates tighter intra-class cohesion (i.e., higher quality features), we formulate the final scale-invariant score as \(T_l^{\text{GRTD}} = - \frac{1}{C}\sum_{c=1}^C \delta_c\), where a larger \(T_l^{\text{GRTD}}\) score corresponds to better transferability.

As we will theoretically justify in Section 3.3, passing frozen embeddings through a randomly initialized decoder spatially aggregates features while preserving their topology. Therefore, we compute the final GRTD score using the feature maps from the decoder.

\textbf{Task-Adaptive Topological Fusion.}
Medical segmentation targets exhibit distinct topological regimes. Multi-organ tasks with high structural complexity
require stronger global structural preservation, whereas small lesion extraction, dominated by boundary-level
ambiguity, benefits more from local topological contrast. To generalize across heterogeneous target distributions, we
introduce a parameter-free task-adaptive fusion mechanism driven by the semantic cardinality of the target task.

We define a task complexity prior as \(\kappa = \log(|\mathcal{C}|)\), where \(|\mathcal{C}|\) is the number of semantic classes. This prior is mapped to a gating factor through a tanh function:
$ 
  \alpha = \tanh\left(\frac{\kappa}{2}\right).
$  
Let \(\mathcal{N}(\cdot)\) denote Min-Max normalization over the model zoo. The final transferability score is
computed as
$
  \mathcal{S}_\phi =
  \alpha \cdot \mathcal{N}(T^{\text{GRTD}})
  +
  (1-\alpha) \cdot \mathcal{N}(T^{\text{LBTC}}).
$
Since \(\alpha\) monotonically increases with task semantic cardinality, complex anatomical segmentation tasks
naturally assign a higher weight to global topological structure, while focal pathology tasks with fewer semantic
classes emphasize local boundary-sensitive topology. This yields a parameter-free adaptive fusion rule across heterogeneous medical segmentation targets.

\subsection{Variance Reduction through Random Upsampling}
\label{subsec:mechanistic_deployment}
A key design choice is determining which layer's features to use for the metric. As we demonstrate in Section 4, computing global alignment on the features of a \textit{randomly initialized} decoder improves transferability estimation. To understand why this works, we model the random decoder as a composition of random linear projections and spatial upsampling. We first show that these random convolutions act as random projections that preserve the underlying feature topology, formalizing this through the Johnson-Lindenstrauss (JL) lemma~\cite{matouvsek2008variants} applied to MSTs.

\begin{theorem}[MST Topology Preservation under Random Projection]
\label{thm:mst_preserve}
Let \(\mathcal{X} = \{x_1, \ldots, x_n\} \subset \mathbb{R}^d\) have all pairwise distances distinct, and define the minimum distance ratio \(r = \min_{(i,j) \neq (k,l)} \|x_k - x_l\|_2 / \|x_i - x_j\|_2 > 1\) for \(\|x_i - x_j\|_2 < \|x_k - x_l\|_2\). 
Let \(f(x) = \frac{1}{\sqrt{m}}\mathbf{A}x\) where \(\mathbf{A} \in \mathbb{R}^{m \times d}\) has i.i.d. \(\mathcal{N}(0,1)\) entries. For any \(\delta \in (0,1)\) and \(\varepsilon < \frac{r - 1}{r + 1}\), choosing \(m \geq \frac{8}{\varepsilon^2}\ln\frac{n(n-1)}{\delta}\) guarantees that the MST edge set is perfectly preserved: \(E\bigl(\text{MST}(f(\mathcal{X}))\bigr) = E(\text{MST}(\mathcal{X}))\) with probability at least \(1 - \delta\).
\end{theorem}

Having established that the topological skeleton survives the random convolutions,
we next explain why the \textit{upsampling} operation specifically drives global
stabilization. Projecting a single latent feature into \(S\) new spatial locations
via a random kernel is mathematically equivalent to expanding the effective
projection dimension from \(m\) to \(Sm\).

\begin{proposition}[Variance Reduction via Upsampling Spatial Expansion]
\label{prop:variance_reduction}
Let \(S\) be the spatial expansion factor of an upsampling layer and
\(\mathbf{A} \in \mathbb{R}^{Sm \times d}\) a random matrix with i.i.d.\
\(\mathcal{N}(0,1)\) entries, modelling the randomly initialised upsampling
projection (in transposed convolutions, the \(S\) distinct kernel-position weight
slices are mutually independent under random initialisation, making this model
exact). Let \(\mathbf{A}_1 \in \mathbb{R}^{m \times d}\) denote any fixed
\(m\)-row submatrix of \(\mathbf{A}\). For any fixed
\(u = x_i - x_j \in \mathbb{R}^d\), the single-location and
spatially-aggregated squared distance estimates are defined as
\(\hat{D}_1 = \frac{1}{m}\|\mathbf{A}_1 u\|_2^2\) and
\(\hat{D}_S = \frac{1}{Sm}\|\mathbf{A}u\|_2^2\), respectively.
Both are unbiased estimators of \(\|u\|_2^2\), but spatial expansion strictly
reduces the variance:
$
    \operatorname{Var}[\hat{D}_S]
    \;=\; \frac{1}{S}\operatorname{Var}[\hat{D}_1]
    \;=\; \frac{2}{Sm}\|u\|_2^4.
$
Consequently, for any \(\varepsilon > 0\), the tail probability bound drops
exponentially with the spatial expansion factor \(S\):
$
    \mathbb{P}\!\left[\,
        \bigl|\hat{D}_S - \|u\|_2^2\bigr| > \varepsilon\|u\|_2^2
    \,\right]
    \;\leq\;
    2\exp\!\left(-\frac{Sm\varepsilon^2}{8}\right).
$
\end{proposition}

Proposition~\ref{prop:variance_reduction} theoretically validates our design:
upsampling makes the global feature structure easier to evaluate. By aggregating
features spatially, it reduces the standard deviation of pairwise distance
estimates by a factor of \(\sqrt{S}\), thereby stabilizing the distance ordering.

\section{Experiments}

\subsection{Setup and Implementation Details}
We evaluate our approach using the OpenMind benchmark~\cite{wald2025openmind}, which provides foundation models pre-trained on 114,000 unlabeled 3D brain MRIs. The model pool features a ResEnc-L backbone~\cite{isensee2024nnu} trained with 7 diverse self-supervised objectives, including reconstruction (MAE~\cite{He2022MAE}, SimMIM~\cite{chen2023masked}, MG~\cite{zhou2021models}) and contrastive (VoCo~\cite{wu2024voco}, SimCLR~\cite{chen2020simple}, SwinUNETR~\cite{Tang2022SwinUNETR}, VoCo~\cite{wu2024voco}) methods. To comprehensively assess transferability, we employ 5 in-distribution (ID) brain and head-and-neck MRI tasks (ISL~\cite{hernandez2022isles}, HNT~\cite{wahid2024hntsmrg}, MSF~\cite{muslim2022ms}, TPC~\cite{yang2024topcow}, YBM~\cite{ramakrishnan2024large}, and 2 out-of-distribution (OOD) tasks featuring anatomical or cross-modal shifts (ACDC~\cite{bernard2018acdc} for cardiac MRI, KiTS19~\cite{heller2021state} for kidney CT). For more details, please refer to the Appendix~\ref{app:setup}.

For each candidate foundation model's encoder, we attach a randomly initialized nnU-Net decoder and extract multi-scale features via sliding-window inference. To evaluate transferability efficiently while mitigating the severe class imbalance inherent to medical images, we employ a class-proportional stratified sampling strategy, randomly extracting 256 foreground and 2,560 background voxels per volume. Crucially, following our theoretical insights, the GRTD is computed on the decoder layers to leverage their spatial unfolding bias, whereas LBTC operates on the encoder layers. To account for the varying representational focus at different depths, we aggregate the per-layer scores for both metrics using a Gaussian weighting strategy centered at their respective mid-level layers. Finally, we measure the model ranking quality using the weighted Kendall's \(\tau^{*,w}\)~\cite{vigna2015weighted}, which assigns higher penalties to discordant pairs near the top of the ranking to reflect the practical priority of model selection. For comprehensive implementation details, including the exact Gaussian weighting formulations, metric hyperparameters, and per-case sampling budgets, please refer to Appendix~\ref{sec:setting_details}.

\begin{table}[t]
\centering
\caption{Weighted Kendall's \(\tau\) for transferability estimation on the segmentation tasks.}
\vspace{-2mm}
\label{tab:main_results}
\small 
\renewcommand{\arraystretch}{0.95} 
\setlength{\tabcolsep}{4pt} 

\begin{tabular}{l ccccc |cc |c}
\toprule
\multirow{2}{*}{\textbf{Method}} &
\multicolumn{5}{c}{\textbf{ID (Same Region)}} &
\multicolumn{2}{c}{\textbf{OOD}} &
\multirow{2}{*}{\textbf{Avg}} \\
\cmidrule(lr){2-6} \cmidrule(lr){7-8} 
& \textbf{MSF} & \textbf{ISL} & \textbf{HNT} & \textbf{TPC} & \textbf{YBM} & \textbf{ACD} & \textbf{KIT} & \\
\midrule
LogME~\cite{you2021logme}           & -0.514 & 0.134 & -0.310 & 0.034 & 0.260 & -0.228 & -0.456 & -0.154 \\
LEEP~\cite{nguyen2020leep}          & -0.524 & -0.313 & -0.216 & -0.109& -0.512 & -0.109 & 0.002 & -0.254 \\
GBC~\cite{pandy2022transferability} & -0.471 & -0.174 & -0.710 & -0.230 & -0.523 & 0.451 & -0.000 & -0.237 \\
CCFV~\cite{yang2023pick}            & 0.103 & 0.104 & 0.407 & 0.162  & 0.153 & \textbf{0.943} & 0.034 & 0.272 \\
Ours                                & \textbf{0.942} & \textbf{0.352} & \textbf{0.842} & \textbf{0.546} & \textbf{0.756} & 0.910 & \textbf{0.116} & \textbf{0.638} \\
\bottomrule
\end{tabular}
\vspace{-3mm} 
\end{table}

\subsection{Main Results}
The evaluation of our topology-driven framework against established Transferability Estimation (TE) metrics using the Weighted Kendall's \(\tau\) are shown in Table \ref{tab:main_results}. The results starkly validate our arguments: traditional classification-based metrics (LogME~\cite{you2021logme}, LEEP~\cite{nguyen2020leep}, GBC~\cite{pandy2022transferability}) consistently yield near-zero or negative correlations across both In-Distribution (ID) and Out-of-Distribution (OOD) datasets. This confirms that simply measuring global class-conditional feature separation misleads model selection for medical segmentation, where complex spatial topologies govern performance.

Compared to CCFV~\cite{yang2023pick}, the only method designed for segmentation tasks, our framework demonstrates superiority. On ID tasks with highly fragmented boundaries, such as MSF and HNT, our method achieves remarkable \(\tau\) scores of 0.942 and 0.842, vastly exceeding CCFV (0.103 and 0.407) and highlighting the efficacy of our Local Boundary-Aware Topological Consistency (LBTC). Furthermore, under severe OOD shifts involving distant anatomical regions and cross-modality transfers, our method maintains robust predictive power. It achieves a highly competitive correlation on ACD (\(\tau=0.910\)) and outperforms all baselines on the challenging KIT dataset (\(\tau=0.116\), vs. CCFV's 0.034). Overall, our framework yields an average \(\tau\) of 0.638 across all tasks—more than double that of CCFV (0.272), proving that our metric provides a highly reliable model selection mechanism across diverse clinical complexities.

\begin{figure*}[tbp]
    \includegraphics[width=0.99\linewidth, height=0.28\linewidth]{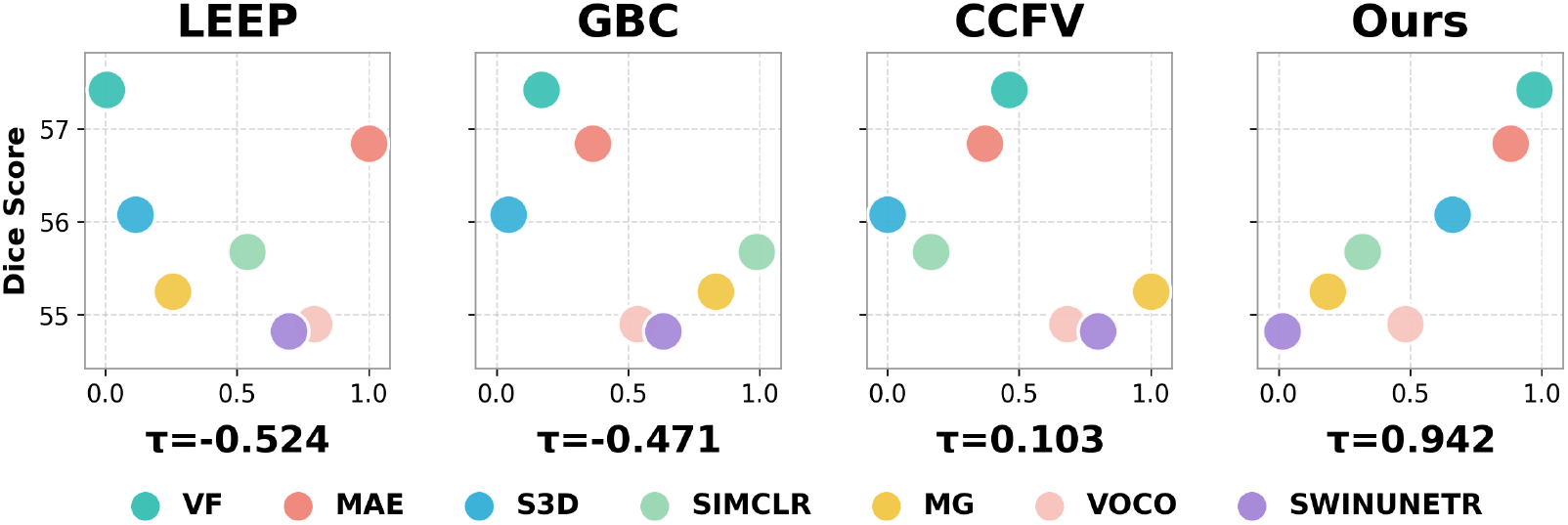} 
    \vspace{-0.2cm}
    \caption{Correlation between the fine-tuning performance and transferability metrics using MSF as an example. The vertical axis represents the average Dice, while the horizontal axis represents the \emph{normalized} transferability metric. }
\label{fig:correlation_matrix}
\vspace{-4mm}
\end{figure*}

\subsection{Analysis}

\noindent\textbf{Computational Time Comparison.}
Table~\ref{tab:time_skip} (left) contrasts the practical cost of model selection via training-free scoring versus exhaustive downstream fine-tuning. While exhaustive fine-tuning across the model zoo requires a prohibitive amount of time (e.g., 3000+ minutes per task), training-free TE methods offer a more viable path. However, existing dense-prediction metrics like CCFV~\cite{yang2023pick} still incur substantial computational overhead (averaging 841.63 seconds per task) due to the heavy burden of computing dense pairwise relationships. Our approach effectively ease this bottleneck, reducing the average evaluation time to merely 15.09 seconds averaged across three different random seeds (for the random seed analysis, please refer to Appendix ~\ref{app:analysis}), achieving a remarkable speedup of approximately 56 times over CCFV. This efficiency stems fundamentally from our topology-driven design. By utilizing the Minimum Spanning Tree (MST), we bypass the need to evaluate the dense correlation graph, efficiently distilling the feature manifold into a sparse and highly effective representation.

\begin{table}[htp]
\centering
\vspace{-3mm}
\caption{
    \textbf{Left:} Time comparison (seconds) across different datasets.
    \textbf{Right:} Decoder vs. Encoder-side Skip Source comparison (ER/PR metrics).
}
\vspace{-2mm}
\label{tab:time_skip}
\footnotesize
\renewcommand{\arraystretch}{0.95} 
\setlength{\tabcolsep}{3pt}       

\begin{minipage}[t]{0.38\textwidth}
\centering
\textbf{Time comparison}\\
\begin{tabular}{lccc}
\toprule
Dataset & CCFV & Ours & Fine-tune \\
\midrule
TPC & 1050.5 & \textbf{30.69} & 3000+ \\
ACD & 828.9 &  \textbf{15.02} & 3000+ \\
MSF & 175.3 &  \textbf{5.62} & 3000+ \\
YBM & 1311.8 &  \textbf{9.03} & 3000+ \\
\midrule
Avg. & 841.63 & 15.09 & 3000+ \\
\bottomrule
\end{tabular}
\end{minipage}
\hfill
\begin{minipage}[t]{0.58\textwidth}
\centering
\textbf{Decoder vs. Encoder-side Skip Source}\\
\begin{tabular}{l ccc ccc}
\toprule
& \multicolumn{3}{c}{\(\Delta \mathrm{ER}\)} & \multicolumn{3}{c}{\(\Delta \mathrm{PR}\)} \\
\cmidrule(lr){2-4}\cmidrule(lr){5-7}
Model & Full & NoSkip & \(\Delta\) & Full & NoSkip & \(\Delta\) \\
\midrule
MG         & 38.1 & \textbf{52.4} & \(+14.3\) & 13.1 & \textbf{38.7} & \(+25.6\) \\
SimCLR     & 36.1 & \textbf{47.5} & \(+11.4\) & 12.4 & \textbf{25.2} & \(+12.8\) \\
SwinUNETR  & 36.8 & \textbf{66.9} & \(+30.1\) &  8.7 & \textbf{39.1} & \(+30.4\) \\
\midrule
Mean       & 37.0 & \textbf{55.6} & \(+18.6\) & 11.4 & \textbf{34.3} & \(+22.9\) \\
\bottomrule
\end{tabular}
\end{minipage}
\vspace{-3mm} 
\end{table}

\noindent\textbf{Random Decoder Analysis}
We validate the decoder's topological roles by measuring changes in Effective Rank (\(\Delta \mathrm{ER}\)) and Participation Ratio (\(\Delta \mathrm{PR}\)) from encoder skip sources to decoder representations. These metrics quantify manifold dimensionality, where increases reflect spatial expansion and reduced feature collapse. As Table~\ref{tab:time_skip} (Right) shows, the standard decoder yields positive \(\Delta \mathrm{ER}\) (mean \(37.0\)) and \(\Delta \mathrm{PR}\) (mean \(11.4\)). Removing skip connections (\textit{NoSkip}) drastically amplifies this expansion (\(\Delta \mathrm{ER}\) surges by \(+18.6\), \(\Delta \mathrm{PR}\) by \(+22.9\)). This corroborates our theoretical insights: pure upsampling drives global manifold unfolding to capture global structural layouts.

\noindent\textbf{The Necessity of Task-Adaptive Fusion}
Medical segmentation targets exhibit distinct topological regimes. As shown in Table~\ref{tab:adaptive_cls}(Left), we stratify tasks into \textit{Fragmented} targets (ISL, MSF) which rely heavily on local boundary contrast, and \textit{Structured} targets (ACD, TPC) which demand global structural preservation. The Global metric (GRTD) excels on structured tasks but fails on fragmented ones, while the Local metric (LBTC) shows the exact opposite trend. Our Task-Adaptive fusion dynamically balances these priors based on semantic cardinality, achieving robust, optimal performance across all anatomical regimes.

\begin{table}[htbp] 
\centering
\caption{
    \textbf{Left:} Adaptive fusion performance across anatomical targets in different dataset. 
    \textbf{Right:} Transferability estimation results on classification tasks (MRN~\cite{bien2018deep}, ABI~\cite{di2014autism}).
}
\label{tab:adaptive_cls}
\vspace{-2mm} 

\small
\renewcommand{\arraystretch}{1.0}
\setlength{\tabcolsep}{3pt}

\begin{minipage}[t]{0.52\textwidth}
\centering
\textbf{Adaptive fusion performance}\\
\vspace{2pt}
\begin{tabular}{l|cc|cc|c}
\toprule
\textbf{Method} & ISL & MSF & ACD & TPC & Average \\
\midrule
GRTD & 0.09 & 0.06 & \textbf{0.94} & \textbf{0.55} & 0.41 \\
LBTC & \textbf{0.35} & \textbf{0.94} & 0.03 & 0.12 & 0.36 \\
\midrule
\textbf{Adaptive} & \textbf{0.35} & \textbf{0.94} & \textbf{0.91} & \textbf{0.55} & \textbf{0.69} \\
\bottomrule
\end{tabular}
\end{minipage}
\hfill
\begin{minipage}[t]{0.45\textwidth}
\centering
\textbf{Our topology-based TE on Classification Task}\\
\vspace{2pt}
\begin{tabular}{lccc}
\toprule
\textbf{Method} & MRN & ABI & Average \\
\midrule
LEEP   & 0.397 & 0.362 & 0.380 \\
LogME  & -0.534 & 0.135 & -0.200 \\
\midrule
Ours   & \textbf{0.407} & \textbf{0.800} & 0.604 \\
\bottomrule
\end{tabular}
\end{minipage}
\vspace{-1mm}
\end{table}

\noindent\textbf{Extension to Classification and Regression}
\label{subsec:classification_extension}
While primarily designed for dense prediction, our framework naturally extends to global tasks (classification and regression) using spatially pooled representations. To handle severe class imbalance, we introduce the \textit{Topological Label Coherence (TLC)} rate, \(\rho_n^{\text{TLC}}\). For classification, TLC acts as a macro-averaged MST leakage rate to ensure equal contribution from minority classes; for regression, it measures normalized continuous label variation along the tree. 
Crucially, \(\rho_n^{\text{TLC}}\) retains strong theoretical guarantees: it analytically lower-bounds the empirical 1-NN risk, which in turn bounds the optimal Bayes risk \(R^\ast\) via the Cover-Hart inequality. This yields a theoretically grounded and efficient transferability score, \(\mathcal{S}^{\text{TLC}}(\phi) = 1 - \rho_n^{\text{TLC}}\). Empirical results for this extension are summarized in Table~\ref{tab:adaptive_cls}(Right).

\noindent\textbf{Pixel Partial Sampling Analysis}
\label{app:fullvs_partial}
To ensure our budget-constrained pixel sampling strategy (256 foreground and 2560 background voxels per case) introduces no systematic bias, we compared it against an exhaustive full-pixel approach (around $10^4$ foreground voxels with a 1:10 background ratio) and tested it on the MSF dataset. Results show that both strategies yielded identical model rankings with negligible absolute score differences. As for more detailed analysis, please refer to Appendix~\ref{app:analysis}.

\section{Conclusion \& Limitations}

We introduce a non-parametric, topology-driven framework for efficient transferability estimation (TE) in 3D medical vision foundation models. To overcome the architectural mismatches and spatial blindness of existing metrics, we evaluate feature-label alignment directly via Minimum Spanning Trees (MST). Our approach decouples evaluation into Local Boundary-Aware Topological Consistency (LBTC) and Global Representation Topology Divergence (GRTD), dynamically fused via a task-adaptive gate. Furthermore, we prove that a randomly initialized decoder acts as a variance-reducing spatial projector, stabilizing global alignment. Ultimately, our method achieves state-of-the-art ranking performance while significantly accelerating evaluation.

Despite these advances, certain limitations remain. First, robust MST construction inherently relies on sufficient target annotations, leaving its efficacy in extreme few-shot regimes an open question. Second, while we formally justify the spatial projection effect for convolution-based upsampling, its generalization to purely attention-based or diffusion-driven decoders warrants deeper investigation. Looking ahead, beyond its primary application in foundation model selection, our topological metrics could potentially evolve into online diagnostic tools for monitoring pre-training trajectories or guiding parameter-efficient fine-tuning. To sum up, this framework paves the way for more interpretable and resource-efficient adaptation of medical vision foundation models.

\newpage
\bibliographystyle{plain}
\bibliography{cite}

@inproceedings{isensee2024nnu,
  title={nnu-net revisited: A call for rigorous validation in 3d medical image segmentation},
  author={Isensee, Fabian and Wald, Tassilo and Ulrich, Constantin and Baumgartner, Michael and Roy, Saikat and Maier-Hein, Klaus and Jaeger, Paul F},
  booktitle={International Conference on Medical Image Computing and Computer-Assisted Intervention},
  pages={488--498},
  year={2024},
  organization={Springer}
}

@inproceedings{chen2023masked,
  title={Masked image modeling advances 3d medical image analysis},
  author={Chen, Zekai and Agarwal, Devansh and Aggarwal, Kshitij and Safta, Wiem and Balan, Mariann Micsinai and Brown, Kevin},
  booktitle={Proceedings of the IEEE/CVF Winter Conference on Applications of Computer Vision},
  pages={1970--1980},
  year={2023}
}

@article{zhou2021models,
  title={Models genesis},
  author={Zhou, Zongwei and Sodha, Vatsal and Pang, Jiaxuan and Gotway, Michael B and Liang, Jianming},
  journal={Medical image analysis},
  volume={67},
  pages={101840},
  year={2021},
  publisher={Elsevier}
}

@inproceedings{chen2020simple,
  title={A simple framework for contrastive learning of visual representations},
  author={Chen, Ting and Kornblith, Simon and Norouzi, Mohammad and Hinton, Geoffrey},
  booktitle={International conference on machine learning},
  pages={1597--1607},
  year={2020},
  organization={PmLR}
}

@article{juodelyte2024dataset,
  title={On dataset transferability in medical image classification},
  author={Juodelyte, Dovile and Ferrante, Enzo and Lu, Yucheng and Singh, Prabhant and Vanschoren, Joaquin and Cheplygina, Veronika},
  journal={arXiv preprint arXiv:2412.20172},
  year={2024}
}

@article{ramakrishnan2024large,
  title={A large open access dataset of brain metastasis 3D segmentations on MRI with clinical and imaging information},
  author={Ramakrishnan, Divya and Jekel, Leon and Chadha, Saahil and Janas, Anastasia and Moy, Harrison and Maleki, Nazanin and Sala, Matthew and Kaur, Manpreet and Petersen, Gabriel Cassinelli and Merkaj, Sara and others},
  journal={Scientific Data},
  volume={11},
  number={1},
  pages={254},
  year={2024},
  publisher={Nature Publishing Group UK London}
}

@article{grovik2020deep,
  title={Deep learning enables automatic detection and segmentation of brain metastases on multisequence MRI},
  author={Gr{\o}vik, Endre and Yi, Darvin and Iv, Michael and Tong, Elizabeth and Rubin, Daniel and Zaharchuk, Greg},
  journal={Journal of Magnetic Resonance Imaging},
  volume={51},
  number={1},
  pages={175--182},
  year={2020},
  publisher={Wiley Online Library}
}

@article{chen2022carotid,
  title={Carotid vessel wall segmentation and atherosclerosis diagnosis challenge},
  author={Chen, Huijun and Zhao, Xihai and Dou, Jiaqi and Du, C and Yang, R and Sun, H and Yu, S and Zhao, H and Yuan, C and Balu, N},
  journal={Runyu Yang, Haozhong Sun, Shuwan Yu, Huilin Zhao, Chun Yuan, and Niranjan Balu,“Carotid vessel wall segmentation and atherosclerosis diagnosis challenge},
  year={2022}
}

@article{di2014autism,
  title={The autism brain imaging data exchange: towards a large-scale evaluation of the intrinsic brain architecture in autism},
  author={Di Martino, Adriana and Yan, Chao-Gan and Li, Qingyang and Denio, Erin and Castellanos, Francisco X and Alaerts, Kaat and Anderson, Jeffrey S and Assaf, Michal and Bookheimer, Susan Y and Dapretto, Mirella and others},
  journal={Molecular psychiatry},
  volume={19},
  number={6},
  pages={659--667},
  year={2014},
  publisher={Nature Publishing Group}
}

@article{bien2018deep,
  title={Deep-learning-assisted diagnosis for knee magnetic resonance imaging: development and retrospective validation of MRNet},
  author={Bien, Nicholas and Rajpurkar, Pranav and Ball, Robyn L and Irvin, Jeremy and Park, Allison and Jones, Erik and Bereket, Michael and Patel, Bhavik N and Yeom, Kristen W and Shpanskaya, Katie and others},
  journal={PLoS medicine},
  volume={15},
  number={11},
  pages={e1002699},
  year={2018},
  publisher={Public Library of Science San Francisco, CA USA}
}

@article{matouvsek2008variants,
  title={On variants of the Johnson--Lindenstrauss lemma},
  author={Matou{\v{s}}ek, Ji{\v{r}}{\'\i}},
  journal={Random Structures \& Algorithms},
  volume={33},
  number={2},
  pages={142--156},
  year={2008},
  publisher={Wiley Online Library}
}

@article{yang2024topcow,
  title={TopCoW: Benchmarking Topology-Aware Anatomical Segmentation of the Circle of Willis (CoW) for CTA and MRA},
  author={Yang, K and Musio, F and Ma, Y and Juchler, N},
  journal={Journal of Medical Imaging Technology},
  volume={34},
  pages={123--135},
  year={2024}
}

@inproceedings{yang2023pick,
  title={Pick the best pre-trained model: Towards transferability estimation for medical image segmentation},
  author={Yang, Yuncheng and Wei, Meng and He, Junjun and Yang, Jie and Ye, Jin and Gu, Yun},
  booktitle={International Conference on Medical Image Computing and Computer-Assisted Intervention},
  pages={674--683},
  year={2023},
  organization={Springer}
}

@inproceedings{bao2019information,
  title={An information-theoretic approach to transferability in task transfer learning},
  author={Bao, Yajie and Li, Yang and Huang, Shao-Lun and Zhang, Lin and Zheng, Lizhong and Zamir, Amir and Guibas, Leonidas},
  booktitle={2019 IEEE international conference on image processing (ICIP)},
  pages={2309--2313},
  year={2019},
  organization={IEEE}
}

@inproceedings{you2021logme,
  title={Logme: Practical assessment of pre-trained models for transfer learning},
  author={You, Kaichao and Liu, Yong and Wang, Jianmin and Long, Mingsheng},
  booktitle={International Conference on Machine Learning},
  pages={12133--12143},
  year={2021},
  organization={PMLR}
}

@inproceedings{tan2021otce,
  title={Otce: A transferability metric for cross-domain cross-task representations},
  author={Tan, Yang and Li, Yang and Huang, Shao-Lun},
  booktitle={Proceedings of the IEEE/CVF conference on computer vision and pattern recognition},
  pages={15779--15788},
  year={2021}
}

@inproceedings{glorot2010understanding,
  title={Understanding the difficulty of training deep feedforward neural networks},
  author={Glorot, Xavier and Bengio, Yoshua},
  booktitle={Proceedings of the thirteenth international conference on artificial intelligence and statistics},
  pages={249--256},
  year={2010},
  organization={JMLR Workshop and Conference Proceedings}
}

@inproceedings{he2015delving,
  title={Delving deep into rectifiers: Surpassing human-level performance on imagenet classification},
  author={He, Kaiming and Zhang, Xiangyu and Ren, Shaoqing and Sun, Jian},
  booktitle={Proceedings of the IEEE international conference on computer vision},
  pages={1026--1034},
  year={2015}
}

@inproceedings{khoba2025peftdiff,
  title={PEFTDiff: Diffusion-Guided Transferability Estimation for Parameter-Efficient Fine-Tuning},
  author={Khoba, Prafful Kumar and Wang, Zijian and Arora, Chetan and Baktashmotlagh, Mahsa},
  booktitle={Proceedings of the IEEE/CVF International Conference on Computer Vision},
  pages={1454--1463},
  year={2025}
}

@inproceedings{gholami2023etran,
  title={Etran: Energy-based transferability estimation},
  author={Gholami, Mohsen and Akbari, Mohammad and Wang, Xinglu and Kamranian, Behnam and Zhang, Yong},
  booktitle={Proceedings of the IEEE/CVF International Conference on Computer Vision},
  pages={18613--18622},
  year={2023}
}

@inproceedings{ding2022pactran,
  title={Pactran: Pac-bayesian metrics for estimating the transferability of pretrained models to classification tasks},
  author={Ding, Nan and Chen, Xi and Levinboim, Tomer and Changpinyo, Soravit and Soricut, Radu},
  booktitle={European Conference on Computer Vision},
  pages={252--268},
  year={2022},
  organization={Springer}
}

@article{wang2025triad,
  title={Triad: Vision Foundation Model for 3D Magnetic Resonance Imaging},
  author={Wang, Shansong and Safari, Mojtaba and Li, Qiang and Chang, Chih-Wei and Qiu, Richard LJ and Roper, Justin and Yu, David S and Yang, Xiaofeng},
  journal={arXiv preprint arXiv:2502.14064},
  year={2025}
}

@inproceedings{nguyen2020leep,
  title={Leep: A new measure to evaluate transferability of learned representations},
  author={Nguyen, Cuong and Hassner, Tal and Seeger, Matthias and Archambeau, Cedric},
  booktitle={International Conference on Machine Learning},
  pages={7294--7305},
  year={2020},
  organization={PMLR}
}

@inproceedings{hatamizadeh2021swin,
  title={Swin unetr: Swin transformers for semantic segmentation of brain tumors in mri images},
  author={Hatamizadeh, Ali and Nath, Vishwesh and Tang, Yucheng and Yang, Dong and Roth, Holger R and Xu, Daguang},
  booktitle={International MICCAI brainlesion workshop},
  pages={272--284},
  year={2021},
  organization={Springer}
}

@misc{wahid2024hntsmrg,
  author       = {Wahid, Kareem and Dede, Cem and Naser, Mohamed and Fuller, Clifton},
  title        = {Training dataset for HNTSMRG 2024 Challenge},
  year         = {2024},
  howpublished = {\url{https://zenodo.org/doi/10.5281/zenodo.11199559}}
}

@article{bernard2018acdc,
  author    = {Bernard, Olivier and Lalande, Alain and Zotti, Clement and Cerv{\'e}nansk{\'y}, J. and others},
  title     = {Deep learning techniques for automatic MRI cardiac multi-structures segmentation and diagnosis: Is the problem solved?},
  journal   = {IEEE Transactions on Medical Imaging},
  year      = {2018}
}

@article{heller2021state,
  title={The state of the art in kidney and kidney tumor segmentation in contrast-enhanced CT imaging: Results of the KiTS19 challenge},
  author={Heller, Nicholas and Isensee, Fabian and Maier-Hein, Klaus H and Hou, Xiaoshuai and Xie, Chunmei and Li, Fengyi and Nan, Yang and Mu, Guangrui and Lin, Zhiyong and Han, Miofei and others},
  journal={Medical image analysis},
  volume={67},
  pages={101821},
  year={2021},
  publisher={Elsevier}
}

@article{muslim2022ms,
  author       = {Muslim, Ali M. and Mashohor, Syamsiah and Al Gawwam, Gheyath and Mahmud, Rozi and Hanafi, Marsyita Binti and Al-nuaimi, Osama and Josephine, Raad and Almutairi, Abdullah Dhaifallah},
  title        = {Brain MRI dataset of multiple sclerosis with consensus manual lesion segmentation and patient meta information},
  journal      = {Data in Brief},
  volume       = {42},
  pages        = {108139},
  year         = {2022},
  issn         = {2352-3409},
  doi          = {10.1016/j.dib.2022.108139},
  url          = {https://www.sciencedirect.com/science/article/pii/S235234092200347X}
}

@article{hernandez2022isles,
  title={ISLES 2022: A multi-center magnetic resonance imaging stroke lesion segmentation dataset},
  author={Hernandez Petzsche, Moritz R and De La Rosa, Ezequiel and Hanning, Uta and Wiest, Roland and Valenzuela, Waldo and Reyes, Mauricio and Meyer, Maria and Liew, Sook-Lei and Kofler, Florian and Ezhov, Ivan and others},
  journal={Scientific data},
  volume={9},
  number={1},
  pages={762},
  year={2022},
  publisher={Nature Publishing Group UK London}
}

@inproceedings{huang2022frustratingly,
  title={Frustratingly easy transferability estimation},
  author={Huang, Long-Kai and Huang, Junzhou and Rong, Yu and Yang, Qiang and Wei, Ying},
  booktitle={International conference on machine learning},
  pages={9201--9225},
  year={2022},
  organization={PMLR}
}

@inproceedings{pandy2022transferability,
  title={Transferability estimation using bhattacharyya class separability},
  author={P{\'a}ndy, Michal and Agostinelli, Andrea and Uijlings, Jasper and Ferrari, Vittorio and Mensink, Thomas},
  booktitle={Proceedings of the IEEE/CVF Conference on Computer Vision and Pattern Recognition},
  pages={9172--9182},
  year={2022}
}

@inproceedings{wald2025openmind,
  title={An OpenMind for 3D medical vision self-supervised learning},
  author={Wald, Tassilo and Ulrich, Constantin and Suprijadi, Jonathan and Ziegler, Sebastian and Nohel, Michal and Peretzke, Robin and Kohler, Gregor and Maier-Hein, Klaus},
  booktitle={Proceedings of the IEEE/CVF International Conference on Computer Vision},
  pages={23839--23879},
  year={2025}
}

@article{isensee2021nnu,
  title={nnU-Net: a self-configuring method for deep learning-based biomedical image segmentation},
  author={Isensee, Fabian and Jaeger, Paul F and Kohl, Simon AA and Petersen, Jens and Maier-Hein, Klaus H},
  journal={Nature methods},
  volume={18},
  number={2},
  pages={203--211},
  year={2021},
  publisher={Nature Publishing Group}
}

@InProceedings{He2022MAE,
    author    = {He, Kaiming and Chen, Xinlei and Xie, Saining and Li, Yanghao and Doll\'ar, Piotr and Girshick, Ross},
    title     = {Masked Autoencoders Are Scalable Vision Learners},
    booktitle = {Proceedings of the IEEE/CVF Conference on Computer Vision and Pattern Recognition (CVPR)},
    month     = {June},
    year      = {2022},
    pages     = {16000-16009}
}

@InProceedings{Tang2022SwinUNETR,
    author    = {Tang, Yucheng and Yang, Dong and Li, Wenqi and Roth, Holger R. and Landman, Bennett and Xu, Daguang and Nath, Vishwesh and Hatamizadeh, Ali},
    title     = {Self-Supervised Pre-Training of Swin Transformers for 3D Medical Image Analysis},
    booktitle = {Proceedings of the IEEE/CVF Conference on Computer Vision and Pattern Recognition (CVPR)},
    month     = {June},
    year      = {2022},
    pages     = {20730-20740}
}

@article{simeoni2025dinov3,
  title={Dinov3},
  author={Sim{\'e}oni, Oriane and Vo, Huy V and Seitzer, Maximilian and Baldassarre, Federico and Oquab, Maxime and Jose, Cijo and Khalidov, Vasil and Szafraniec, Marc and Yi, Seungeun and Ramamonjisoa, Micha{\"e}l and others},
  journal={arXiv preprint arXiv:2508.10104
        
        
        
        
        
        
        
        },
  year={2025}
}

@inproceedings{vigna2015weighted,
  title={A weighted correlation index for rankings with ties},
  author={Vigna, Sebastiano},
  booktitle={Proceedings of the 24th international conference on World Wide Web},
  pages={1166--1176},
  year={2015}
}

@article{huang2022green,
  title={Green hierarchical vision transformer for masked image modeling},
  author={Huang, Lang and You, Shan and Zheng, Mingkai and Wang, Fei and Qian, Chen and Yamasaki, Toshihiko},
  journal={Advances in Neural Information Processing Systems},
  volume={35},
  pages={19997--20010},
  year={2022}
}

@inproceedings{xiao2023delving,
  title={Delving into masked autoencoders for multi-label thorax disease classification},
  author={Xiao, Junfei and Bai, Yutong and Yuille, Alan and Zhou, Zongwei},
  booktitle={Proceedings of the IEEE/CVF winter conference on applications of computer vision},
  pages={3588--3600},
  year={2023}
}

@article{bolya2021scalable,
  title={Scalable diverse model selection for accessible transfer learning},
  author={Bolya, Daniel and Mittapalli, Rohit and Hoffman, Judy},
  journal={Advances in Neural Information Processing Systems},
  volume={34},
  pages={19301--19312},
  year={2021}
}

@article{mh2023lvm,
  title={Lvm-med: Learning large-scale self-supervised vision models for medical imaging via second-order graph matching},
  author={MH Nguyen, Duy and Nguyen, Hoang and Diep, Nghiem and Pham, Tan Ngoc and Cao, Tri and Nguyen, Binh and Swoboda, Paul and Ho, Nhat and Albarqouni, Shadi and Xie, Pengtao and others},
  journal={Advances in Neural Information Processing Systems},
  volume={36},
  pages={27922--27950},
  year={2023}
}

@article{ji2022amos,
  title={Amos: A large-scale abdominal multi-organ benchmark for versatile medical image segmentation},
  author={Ji, Yuanfeng and Bai, Haotian and Ge, Chongjian and Yang, Jie and Zhu, Ye and Zhang, Ruimao and Li, Zhen and Zhanng, Lingyan and Ma, Wanling and Wan, Xiang and others},
  journal={Advances in neural information processing systems},
  volume={35},
  pages={36722--36732},
  year={2022}
}

@article{podobnik2023han,
  title={HaN-Seg: The head and neck organ-at-risk CT and MR segmentation dataset},
  author={Podobnik, Ga{\v{s}}per and Strojan, Primo{\v{z}} and Peterlin, Primo{\v{z}} and Ibragimov, Bulat and Vrtovec, Toma{\v{z}}},
  journal={Medical physics},
  volume={50},
  number={3},
  pages={1917--1927},
  year={2023},
  publisher={Wiley Online Library}
}

@article{hero1999asymptotic,
  title={Asymptotic theory of greedy approximations to minimal k-point random graphs},
  author={Hero, Alfred O and Michel, Olivier JJ},
  journal={IEEE Transactions on Information Theory},
  volume={45},
  number={6},
  pages={1921--1938},
  year={1999},
  publisher={IEEE}
}

@inproceedings{li2021ranking,
  title={Ranking neural checkpoints},
  author={Li, Yandong and Jia, Xuhui and Sang, Ruoxin and Zhu, Yukun and Green, Bradley and Wang, Liqiang and Gong, Boqing},
  booktitle={Proceedings of the IEEE/CVF Conference on Computer Vision and Pattern Recognition},
  pages={2663--2673},
  year={2021}
}

@article{perez2025exploring,
  title={Exploring scalable medical image encoders beyond text supervision},
  author={P{\'e}rez-Garc{\'\i}a, Fernando and Sharma, Harshita and Bond-Taylor, Sam and Bouzid, Kenza and Salvatelli, Valentina and Ilse, Maximilian and Bannur, Shruthi and Castro, Daniel C and Schwaighofer, Anton and Lungren, Matthew P and others},
  journal={Nature Machine Intelligence},
  volume={7},
  number={1},
  pages={119--130},
  year={2025},
  publisher={Nature Publishing Group UK London}
}

@article{filiot2024phikon,
  title={Phikon-v2, a large and public feature extractor for biomarker prediction},
  author={Filiot, Alexandre and Jacob, Paul and Mac Kain, Alice and Saillard, Charlie},
  journal={arXiv preprint arXiv:2409.09173
        
        },
  year={2024}
}

@article{zedda2026omnirad,
  title={OmniRad: A Radiological Foundation Model for Multi-Task Medical Image Analysis},
  author={Zedda, Luca and Loddo, Andrea and Di Ruberto, Cecilia},
  journal={arXiv preprint arXiv:2602.04547
        
        },
  year={2026}
}

@inproceedings{he2025vista3d,
  title={VISTA3D: A unified segmentation foundation model for 3D medical imaging},
  author={He, Yufan and Guo, Pengfei and Tang, Yucheng and Myronenko, Andriy and Nath, Vishwesh and Xu, Ziyue and Yang, Dong and Zhao, Can and Simon, Benjamin and Belue, Mason and others},
  booktitle={Proceedings of the Computer Vision and Pattern Recognition Conference},
  pages={20863--20873},
  year={2025}
}

@inproceedings{tran2019transferability,
  title={Transferability and hardness of supervised classification tasks},
  author={Tran, Anh T and Nguyen, Cuong V and Hassner, Tal},
  booktitle={Proceedings of the IEEE/CVF international conference on computer vision},
  pages={1395--1405},
  year={2019}
}

@inproceedings{wu2024voco,
  title={Voco: A simple-yet-effective volume contrastive learning framework for 3d medical image analysis},
  author={Wu, Linshan and Zhuang, Jiaxin and Chen, Hao},
  booktitle={Proceedings of the IEEE/CVF conference on computer vision and pattern recognition},
  pages={22873--22882},
  year={2024}
}


\appendix

\section{The Foundation Model Zoo and Downstream Tasks}
\label{app:setup}
We utilize the extensive library of pre-trained models provided by the OpenMind benchmark~\cite{wald2025openmind}. In contrast to prior transferability evaluation studies that used models trained on labeled datasets (e.g., MSD), our source models are trained on 114,000 unlabeled 3D brain MRI volumes. The candidate model pool includes a variety of architectures and self-supervised learning objectives.
Specifically, the architectures are represented by ResEnc-L~\cite{isensee2024nnu}, a CNN-based backbone.
In terms of SSL Objectives, we consider Reconstruction-based methods (MAE\cite{He2022MAE}, SimMIM\cite{chen2023masked}, ModelsGenesis\cite{zhou2021models}) and Contrastive methods (VoCo\cite{wu2024voco}, SimCLR\cite{chen2020simple}, SwinUNETR\cite{Tang2022SwinUNETR}).
This setup rigorously evaluates the capability of TE metrics to assess foundation models, where the encoder has never been exposed to segmentation labels during pre-training.To comprehensively evaluate the transferability of foundation models (pre-trained on head-and-neck images), we also implement the diverse test suite defined in the OpenMind benchmark~\cite{wald2025openmind}. This suite encompasses 5 in-distribution and 2 out-of-distribution tasks across various anatomical regions (Brain, Head \& Neck, Cardiac, Kidney) and modalities (MRI, CT). 

\noindent
\textbf{In-Distribution (ID) Tasks}
Our primary evaluation focuses on 3 distinct in-distribution datasets, aligning with (or closely related with) the pre-training domain. These tasks serve as a critical baseline to quantify the model's intrinsic transferability.
\textbf{ISLES (ISL)}~\cite{hernandez2022isles}: Stroke lesion segmentation using DWI ($b=1000$) and ADC map brain MRI scans. As a central nervous system (CNS) task, it shares anatomical similarity with the head-and-neck region, testing transfer to a near-ID pathological task (with clear lesion boundaries). 
\textbf{HNTS-MRG (HNT)}~\cite{wahid2024hntsmrg}: Segmentation of head-and-neck primary tumors and metastatic lymph nodes (pre-treatment MR). Directly match the pre-training anatomical region, serving as the most aligned ID task to test basic transferability. 
\textbf{MS FLAIR (MSF)}~\cite{muslim2022ms}: Segmentation of hyperintense multiple-sclerosis lesions on FLAIR MRIs. It is another CNS task used to validate transferability to a different brain pathology (with hyperintense lesion boundaries). \textbf{ToP-CoW (TPC)}~\cite{yang2024topcow}: Arterial structure segmentation on Time-of-Flight (ToF) Magnetic Resonance Angiography (MRA) of the Circle of Willis. As a CNS vascular task anatomically close to the head-and-neck region, it tests transferability to a near-ID neurovascular setting with well-defined vessel boundaries. \textbf{Yale Brain Metastasis (YBM)}~\cite{ramakrishnan2024large}: Segmentation of brain metastases using pre- and post-contrast T1, T2, and FLAIR MRI images. As a challenging neuro-oncology task, it evaluates the model's capacity to delineate metastatic lesions across multi-parametric MRI sequences.

\noindent
\textbf{Out-of-Distribution (OOD) Tasks}
To further assess the generalization capacity of the representations beyond pre-training regions, we also consider 2 OOD tasks to cover diverse distribution shifts.
\textbf{ACDC (ACD)}~\cite{bernard2018acdc}: Cardiac cine-MRI segmentation of three ventricular structures (Right Ventricle, Left Ventricle, and Myocardium). Tests transfer to the thoracic region (OOD), focusing on boundary-rich anatomical structures (ventricles). 
\textbf{KiTS19 (KIT)}~\cite{heller2021state}: CT images of kidneys and kidney tumors, which is the most distant OOD task (abdominal region, cross-modality MR→CT) but was found to be a robust benchmarking dataset in \cite{isensee2024nnu}. Its clear tumor boundaries make it ideal for testing the models’ generalization to unseen anatomical boundaries.

\section{Setting Details}
\label{sec:setting_details}

\noindent\textbf{Multi-Scale Feature Extraction.}
For each candidate encoder $\phi_k$, we attach a randomly initialized nnU-Net~\cite{isensee2021nnu} decoder and perform sliding-window inference to extract features from 10 layers spanning both the encoder and decoder. Dataset-specific ROI sizes follow nnU-Net planning conventions (e.g., $96{\times}128{\times}128$). Feature vectors at pre-selected voxel coordinates are gathered via nearest-neighbor mapping from input space to each layer's spatial resolution.

\noindent\textbf{Stratified Sampling.}
To mitigate the severe class imbalance inherent to medical segmentation, we employ a class-proportional stratified sampling strategy. For each layer, a foreground budget of $B_\ell{=}256$ voxels per case is allocated across all foreground classes in proportion to their voxel counts. If a minority class is exhausted, its surplus budget is redistributed to remaining classes proportionally to their residual capacity, preventing any single dominant class from monopolizing the sample pool. For metrics that require background context (e.g., LBTC), an additional $B_{\mathrm{bg}}{=}2560$ background voxels are sampled. Similarly, metrics such as CCFV that compute a global Feature Variety ($F_v$) score draw a separate set of 256 voxels uniformly at random across the entire volume, independent of class membership.

\noindent\textbf{Metric Hyperparameters.}
The inter-class penalty $\lambda$ in Eq.~(2) is set adaptively as the median of pairwise feature distances.   For LBTC boundary mining, the 30\% rule is applied to \emph{points}: we keep the closest 30\% foreground points
to background and the closest 30\% background points to foreground as boundary candidates, then sample an
absolute number of anchors ($N=\texttt{num\_boundary\_patches}$, base $N=50$) and build local patches with
$w=\texttt{patch\_size}$ nearest neighbors (base $w=16$).

\noindent\textbf{Evaluation Metric.}
We adopt the weighted Kendall's $\tau^{*,w}$~\cite{vigna2015weighted} to evaluate ranking quality. Unlike the standard Kendall's $\tau$, this variant assigns higher penalty to discordant pairs near the top of the ground-truth ranking, reflecting the practical priority of correctly identifying the best-performing models. Specifically, each model at ground-truth position $i$ receives a weight $w(i) = 1/(i+1)$, and the pairwise concordance is computed as:
\begin{equation}
    \tau^{*,w} = 1 - \frac{2 \sum_{i<j} w(i)\, w(j)\, \mathbf{1}[\text{discordant}(i,j)]}{\sum_{i<j} w(i)\, w(j)},
\end{equation}
where the sums range over all pairs ordered by the ground-truth performance. This top-heavy weighting is particularly suited for model selection, where the cost of misranking the top candidates far exceeds errors among lower-performing models.

\noindent\textbf{Sliding-window feature extraction.}
All features are extracted via an \emph{overlap-0} sliding-window strategy that tiles the input volume with non-overlapping patches of size $(\text{roi}_Z, \text{roi}_Y, \text{roi}_X)$.
For each patch, a single forward pass through the frozen pre-trained model is performed and the intermediate activations at the designated hook layers are recorded.
If any spatial dimension of the input volume is smaller than the corresponding roi size, the volume is symmetrically zero-padded; the padded region is assigned a sentinel label of $-1$ and excluded from all downstream computations.

\noindent\textbf{Per-case sampling budget.}
Given the extracted feature map of shape $(C, Z', Y', X')$, the sampling procedure is as follows.
\begin{enumerate}
    \item \textbf{Label downsampling.} The voxel-level segmentation mask is nearest-neighbour interpolated to the feature-map resolution to obtain a per-voxel label array $\mathbf{y} \in \{-1, 0, 1, \dots, C_{\mathrm{seg}}\}^{Z'Y'X'}$, where $-1$ denotes the padded region and is excluded from all computations.
    \item \textbf{Foreground sampling.} Voxels with label $\geq 1$ are treated as foreground. Exactly \textbf{256 foreground voxels} are drawn uniformly at random per case, stratified across foreground classes so that each class receives an approximately equal quota.
    \item \textbf{Background sampling.} Voxels with label $= 0$ are subsampled to at most \textbf{2560 background voxels} per case. For metrics that require class-distribution estimation (GBC, LogME, GRTD), background samples are retained and treated as class $c=0$. For LEEP, background samples are excluded (\texttt{bg\_sample\_num}$= 0$) since background voxels carry no semantic label.
    \item \textbf{Global (Fv) sampling.} A separate global pool of voxels (foreground and background jointly, up to \texttt{fv\_sample\_num}) is maintained per case for the Feature Variety component of CCFV.
\end{enumerate}

\noindent\textbf{Experiments compute resources}
All experiments were conducted on a workstation equipped with a single NVIDIA RTX 5090 GPU (32GB VRAM), an Intel\textsuperscript{\tiny\textregistered} Xeon\textsuperscript{\tiny\textregistered} Platinum 8470Q CPU (25 vCPUs), and 90GB of system memory. The execution time for the forward pass on the dataset ranges from 10 to 60 minutes.

\section{Adaptation of Classification-oriented Metrics to Dense Prediction}
\label{app:metric_adapt}

LogME and LEEP were originally proposed for image-level classification, where each sample \(x_i\) carries a single label \(y_i\).
Adapting them to 3-D medical image segmentation requires resolving two structural mismatches: (i) each volume produces thousands of labelled feature vectors rather than one, and (ii) the feature extractor is an encoder--decoder network rather than a pure classification backbone.

\noindent\textbf{LogME: feature layer and evidence computation.}
Each sampled voxel \((\mathbf{f}_i, y_i)\) is mapped to a one-hot label vector \(\mathbf{e}_i \in \{0,1\}^C\), and the full feature matrix \(\mathbf{F} \in \mathbb{R}^{N \times d}\) and label matrix \(\mathbf{Y} \in \{0,1\}^{N \times C}\) are assembled across all sampled voxels and cases.
Following Algorithm~1 of \cite{you2021logme}, we compute the truncated SVD of \(\mathbf{F}\) and obtain eigenvalues \(\boldsymbol{\lambda}\) of \(\mathbf{F}^\top\mathbf{F}\) (not their square roots).
The LogME score is
\begin{equation}
    \mathrm{LogME}(\mathbf{F}, \mathbf{Y})
    = \frac{1}{C \cdot N}\sum_{c=1}^{C} \log p(\mathbf{y}_c \mid \mathbf{F}),
\end{equation}
where \(N\) is the total number of sampled voxels (not the per-class count), ensuring normalisation is consistent across classes with different cardinalities.

\noindent\textbf{LEEP: two operating modes.}
LEEP~\citep{nguyen2020leep} requires a source model \(\theta\) that produces a probability distribution over a source label set \(\mathcal{Z}\).
For segmentation target tasks (MSF, ISL, HNT, TPC, ACD, KIT), the deepest decoder segmentation head (\texttt{decoder.seg\_layers.4}) outputs per-voxel logits of dimension \(|\mathcal{Z}| = C_{\mathrm{seg}}\).
We apply softmax to obtain \(\theta(x_i) \in \Delta^{|\mathcal{Z}|-1}\) and follow the three-step LEEP formula:
\begin{align}
    \hat{P}(y, z) &= \frac{1}{n}\sum_{i:\,y_i = y} \theta(x_i)_z, \\[1ex]
    \hat{P}(y \mid z) &= \frac{\hat{P}(y, z)}{\hat{P}(z)}, \\[1ex]
    \mathrm{LEEP}(\theta, \mathcal{D}) &= \frac{1}{n}\sum_{i=1}^{n}\log\!\left(\sum_{z \in \mathcal{Z}} \hat{P}(y_i \mid z)\,\theta(x_i)_z\right).
\end{align}
Background voxels are excluded (\texttt{bg\_sample\_num}=0) since they carry no meaningful semantic segmentation label.
For classification target tasks (ABI, MRN), no trained classification head is available on the source model.
Following the extension proposed in~\cite{nguyen2020leep}, we treat the encoder feature vector itself as the logit vector, setting \(|\mathcal{Z}| = d\):
\begin{equation}
    \theta(x_i) = \mathrm{softmax}\!\left(\mathbf{f}_i^{\texttt{enc}}\right), \qquad \mathbf{f}_i^{\texttt{enc}} \in \mathbb{R}^d.
\end{equation}
This approach avoids introducing a randomly initialised classification head, which would destroy the inter-model differences encoded in the pre-trained features.
For both segmentation and classification tasks, voxel-level feature--label pairs from all cases are pooled into a single dataset \(\mathcal{D}\) before computing the joint distribution \(\hat{P}(y,z)\), ensuring the empirical conditional distribution is estimated from the full target distribution rather than per-volume statistics.

\section{Analysis}
\label{app:analysis}

\noindent\textbf{Robustness to Decoder Initialization and Sampling.}
Although the decoder is randomly initialized and the evaluation involves stratified sampling, our metric should primarily reflect the intrinsic quality of the \emph{pre-trained} representation rather than stochasticity from initialization or sampling. As summarized in Table~\ref{tab:init} and Table~\ref{tab:init2}, the weighted Kendall's \(\tau^{*}_{w}\) remains consistently stable across Kaiming~\cite{he2015delving}/Xavier~\cite{glorot2010understanding}/Gaussian initializations, as well as across different random seeds. The negligible variance across these diverse settings confirms that the topology-driven signal robustly dominates over initialization noise and sampling randomness.

\begin{table}[htbp]
\centering
\caption{The weighted Kendall's \(\tau^{*}_{w}\) across different initialization methods.}
\label{tab:init}
\begin{tabular}{lccc}
\toprule
Dataset & Kaiming & Xavier & Gaussian \\ 
\midrule
TPC & 0.546 & 0.546 & 0.546 \\ 
ACD & 0.910 & 0.916 & 0.916 \\ 
\midrule
Avg. & 0.728 & 0.731 & 0.731 \\ 
\bottomrule
\end{tabular}
\end{table}

\begin{table}[htbp]
\centering
\caption{The weighted Kendall's \(\tau^{*}_{w}\) across different seeds for sampling and decoder initialization.}
\label{tab:init2}
\begin{tabular}{lccc}
\toprule
Dataset & Seed=0 & Seed=1024 & Seed=4096 \\ 
\midrule
TPC & 0.546 & 0.546 & 0.546 \\ 
ACD & 0.916 & 0.916 & 0.916 \\ 
\midrule
Avg. & 0.731 & 0.731 & 0.731 \\ 
\bottomrule
\end{tabular}
\end{table}

\noindent\textbf{Sensitivity Analysis of Gaussian Layer-Weighting.}
We define the layer weights by a normalized Gaussian kernel
$w_i = \exp\bigl(-(i-\mu)^2/2\sigma^2\bigr) / \sum_j \exp\bigl(-(j-\mu)^2/2\sigma^2\bigr)$,
centered at $\mu_d{=}2$ for the five decoder layers and $\mu_e{=}3$
for the six encoder layers (stem excluded).
The default \texttt{paper} profile uses $(\sigma_d, \sigma_e){=}(0.80, 0.90)$,
which corresponds to the normalized coefficients
$[0.002, 0.022, 0.229, 0.499, 0.229, 0.022]$ (decoder) and
$[0.002, 0.038, 0.240, 0.444, 0.240, 0.038]$ (encoder).
We compare three broader parametric variants:
\texttt{robust} $(1.2, 1.5)$, \texttt{wide} $(1.6, 2.0)$,
and the degenerate \texttt{flat} (uniform weights).
As shown in Table~\ref{tab:gaussian_fusion}, all Gaussian-like profiles produce consistent rankings across datasets, with average $\tau^{*,w}$ differences within $0.02$ among the four peaked variants. Only \texttt{flat}, which discards all layer-depth inductive bias, leads to a notable degradation on HNT ($-0.633$). These results confirm that our framework is robust to moderate perturbations of the Gaussian weighting shape, and the \texttt{paper} default is a reliable choice across heterogeneous segmentation tasks.

\begin{table}[h]
\centering
\caption{Weighted Kendall $\tau^{*,w}$ under different Gaussian
         layer-weighting profiles. \texttt{paper} is the default.}
\label{tab:gaussian_fusion}
\setlength{\tabcolsep}{10pt}
\begin{tabular}{lcccc|c}
\hline
\textbf{Profile} & $(\sigma_d, \sigma_e)$ & \textbf{MSF} & \textbf{ISL} & \textbf{HNT} & \textbf{Avg.} \\
\hline
\texttt{paper}  & (0.8, 0.9) & \textbf{0.9424} & 0.3515          & \textbf{0.8422} & 0.7120 \\
\texttt{sharp}  & (0.8, 1.0)     & \textbf{0.9424} & 0.3881          & \textbf{0.8422} & 0.7242 \\
\texttt{robust} & (1.2, 1.5)     & \textbf{0.9424} & 0.4648          & 0.7846          & \textbf{0.7306} \\
\texttt{wide}   & (1.6, 2.0)     & \textbf{0.9424} & 0.4648          & 0.7846          & \textbf{0.7306} \\
\texttt{flat}   & ($\infty$)     & 0.7926          & \textbf{0.7207} & 0.2090          & 0.5741 \\
\hline
\end{tabular}
\end{table}

\noindent\textbf{Local Boundary Sampling Analysis.}
We conduct a sensitivity analysis of the three key hyperparameters governing LBTC: the number of boundary anchor patches $N$, the local patch size $w$, and the per-patch point budget. The base setting is $(N{=}50,\, w{=}16,\, \text{budget}{=}200)$.

\begin{table}[h]
\centering
\caption{OAT sensitivity of LBTC ($\tau^{*,w}$; base setting in \textbf{bold}).}
\label{tab:lbtc_oat}
\setlength{\tabcolsep}{8pt}
\begin{tabular}{llccc|c}
\hline
\textbf{Hyperparameter} & \textbf{Value} & \textbf{ACD} & \textbf{MSF} & \textbf{YBM} & \textbf{Avg.} \\
\hline
\multirow{3}{*}{$N$ (boundary patches)}
  & 30              & 0.429 & 0.429 & 0.238 & 0.365 \\
  & \textbf{50}     & 0.429 & 0.333 & 0.524 & 0.429 \\
  & 80              & 0.619 & 0.048 & 0.143 & 0.270 \\
\hline
\multirow{3}{*}{$w$ (patch size)}
  & 8               & 0.238 & 0.238 & 0.619 & 0.365 \\
  & \textbf{16}     & 0.429 & 0.333 & 0.524 & 0.429 \\
  & 32              & 0.714 & 0.333 & 0.429 & 0.492 \\
\hline
\multirow{3}{*}{Per-patch budget}
  & 100   & 0.429 & 0.333 & 0.524 & 0.429 \\
  & \textbf{200}             & 0.429 & 0.333 & 0.524 & 0.429 \\
  & 300             & 0.429 & 0.333 & 0.524 & 0.429 \\
\hline
\end{tabular}
\end{table}

As shown in Table~\ref{tab:lbtc_oat}, the per-patch point budget is entirely insensitive across the tested range (100--300), confirming that 100 points already saturates local MST construction. The patch size $w$ and number of boundary patches $N$ show moderate task-dependent variation, but their overall sensitivity ranges are small (0.127 and 0.159 respectively), and the base setting $(N{=}50, w{=}16)$ achieves the best average $\tau^{*,w}$ among all $N$ variants and competitive performance across patch sizes.
These results confirm that LBTC is robust to reasonable perturbations of its sampling hyperparameters.

\noindent\textbf{Global Alignment Normalization Analysis.} A potential concern with the GRTD formulation is that the absolute weight discrepancy $|W(\mathcal{T}_{\text{feat}}) -
W(\mathcal{T}_{\text{sem}})|$ may introduce sensitivity to the absolute feature scale, as encoders with different pre-training objectives (e.g., reconstruction-based MAE vs.\ contrastive methods) can differ substantially in feature magnitude.
A natural remedy is to normalize by the feature tree weight:
\begin{equation}
T^{\text{GRTD}}_{\text{norm}} = -\frac{1}{D}\sum_{d=1}^{D}
\frac{|W(\mathcal{T}^d_{\text{feat}}) - W(\mathcal{T}^d_{\text{sem}})|}{W(\mathcal{T}^d_{\text{feat}})},
\end{equation}
where $T^{\text{GRTD}}_{\text{norm}} \in [-1, 0]$, with $0$ indicating perfect alignment and $-1$ indicating complete misalignment. We validate this variant on YBM and MSF, and find that the weighted Kendall $\tau^{*,w}$ remains identical under both formulations. We therefore retain the original formulation for simplicity.

\noindent\textbf{Fusion Weight Sensitivity Analysis.}
We evaluate the effect of fixing the fusion gate $\alpha$ to
constant values $\{0.3, 0.5, 0.7\}$, and compare against our
task-adaptive formulation $\alpha = \tanh(\kappa/2)$ where
$\kappa = \log(|\mathcal{C}|)$.
In all settings, GRTD is computed on decoder stages and LBTC
on encoder stages (stem excluded).

\begin{table}[h]
\centering
\caption{Weighted Kendall $\tau^{*,w}$ under fixed vs.\ adaptive
         fusion gate $\alpha$. Our task-adaptive $\alpha$ consistently
         matches or outperforms the best fixed value on each dataset.}
\label{tab:alpha_ablation}
\setlength{\tabcolsep}{10pt}
\begin{tabular}{lccc|c}
\hline
\textbf{$\alpha$} & \textbf{MSF} & \textbf{ISL} & \textbf{HNT} & \textbf{Avg.} \\
\hline
$0.3$ (fixed)              & \textbf{0.9744}          & 0.1450          & 0.8038    & 0.6411      \\
$0.5$ (fixed)              & 0.2945          & 0.1706          & \textbf{0.8422}    &  0.4358      \\
$0.7$ (fixed)              & 0.3201          & \textbf{0.3603}          & 0.4072     &0.3625     \\
Adaptive $\tanh(\kappa/2)$ & 0.9424 & 0.3515 & \textbf{0.8422}  &\textbf{0.7120}  \\
\hline
\end{tabular}
\end{table}

As shown in Table~\ref{tab:alpha_ablation}, the optimal fixed $\alpha$ is strongly dataset-dependent: $\alpha{=}0.3$ favours MSF while degrading ISL, and $\alpha{=}0.7$ improves ISL at the cost of a severe drop on HNT. No single fixed value achieves consistently strong performance across
all tasks simultaneously. In contrast, our task-adaptive $\alpha$ automatically adjusts to the semantic cardinality of each task, recovering near-optimal performance across all three datasets without any manual tuning.

\noindent\textbf{Strict Jacobian NTK Analysis.} Our layer-wise analysis indicated a non-monotonic trend where transferability estimation peaks at intermediate decoder layers. To explain this, we computed the decoder-induced empirical Jacobian NTK ($\Theta_l$). As shown in Table~\ref{tab:strict_ntk_decoder_alignment}, the inverse condition number $1/\kappa(\Theta_l)$ reaches its maximum at \textbf{Dec-2}, which exactly coincides with the peak performance of GRTD ($\tau=0.192$). This provides profound theoretical evidence that the random decoder improves the linearized geometry (conditioning) of the representation up to an intermediate stage, after which excessive unscaled random projections cause signal propagation to become ill-conditioned.

\begin{table}[htbp]
\centering
\caption{Alignment between strict Jacobian NTK conditioning ($1/\kappa$) and metric performance ($\tau$) across decoder stages. The NTK conditioning peak perfectly matches the GRTD peak.}
\label{tab:strict_ntk_decoder_alignment}
\small
\begin{tabular}{lccc}
\toprule
\textbf{Decoder Stage} & \textbf{Mean} $1/\kappa(\Theta_l)$ $\uparrow$ & \textbf{GRTD} $\tau$ & \textbf{LBTC} $\tau$ \\
\midrule
Dec-0 (Deepest) & 0.0131 & -0.009 & 0.401 \\
Dec-1 & 0.0263 & -0.137 & \textbf{0.618} \\
Dec-2 & \textbf{0.0316} & \textbf{0.192} & 0.228 \\
Dec-3 & 0.0246 & -0.137 & -0.252 \\
Dec-4 (Shallowest) & 0.0182 & 0.055 & -0.273 \\
\bottomrule
\end{tabular}
\end{table}

\section{Detailed Theoretical Proofs}
\label{sec:appendix_proofs}

In this section, we provide the rigorous mathematical proofs for the theoretical claims presented in the main text. Section~\ref{app:leakage_proof} establishes the relationship between the MST leakage rate and the Bayes classification error
while Section~\ref{app:mst_proof} details the proof of MST topology preservation under random projections.

\subsection{Proof of Theorem \ref{thm:leakage_bayes}: MST Leakage Rate and Bayes Error}
\label{app:leakage_proof}

Let \((x_1, y_1), \ldots, (x_n, y_n)\) be i.i.d. samples from a distribution \(P\) on \(\mathbb{R}^d \times \mathcal{C}\). The empirical 1-NN error is \(R_n^{\text{1NN}} = \frac{1}{n}\sum_{i=1}^n \mathbf{1}\bigl[y_i \neq y_{\text{NN}(i)}\bigr]\), and the Bayes error is \(R^* = \inf_{g}\mathbb{P}[g(X) \neq Y]\).

\paragraph{Proof of Theorem \ref{thm:leakage_bayes}}
The proof is divided into three parts: finite-sample bound, asymptotic Cover-Hart sandwich, and consistency.

\textbf{Part (i) Finite-sample lower bound.} 
Since $P_X$ is absolutely continuous with respect to the Lebesgue measure, for any two distinct index pairs $(i,j)\neq(k,l)$ the probability that $\|x_i - x_j\|_2 = \|x_k - x_l\|_2$ is zero; a union bound over the finitely many $\binom{n}{2}$ pairs therefore yields that all pairwise distances are almost surely distinct, so $T_n$ is the unique MST and Lemma~E.2 applies almost surely.
For each \(i \in [n]\), define the cross-class indicator \(A_i = \mathbf{1}[y_i \neq y_{\text{NN}(i)}]\), so that \(\sum_i A_i = nR_n^{\text{1NN}}\). By Lemma~\ref{lem:1nn}, the undirected edge \(e_i := \{i,\, \text{NN}(i)\}\) belongs to \(E(T_n)\) for every \(i\). 

Notice that each undirected edge appears at most \textit{twice} in the multiset \(\{e_i\}_{i=1}^n\) (once as \(e_i\) when \(\text{NN}(i) = j\), and once as \(e_j\) when \(\text{NN}(j) = i\) in a mutual nearest-neighbor pair). Hence, the number of \textit{distinct} cross-class edges present in the MST, denoted as \(|E_\times(T_n)|\), satisfies:
\[
  |E_\times(T_n)| \geq \bigl|\bigl\{e_i : A_i = 1\bigr\}\bigr| \geq \frac{1}{2}\sum_{i=1}^n A_i = \frac{n}{2}R_n^{\text{1NN}}
\]
The MST leakage rate is defined as \(\rho_n = \frac{|E_\times(T_n)|}{n-1}\). Dividing both sides by the tree size \(n - 1\) yields the finite-sample lower bound:
\[
  \rho_n \geq \frac{n}{2(n-1)}R_n^{\text{1NN}}
\]

\textbf{Part (ii) Asymptotic Bayes error sandwich.} 
The lower bound \(R^* \leq R^{\text{1NN}}\) is immediate from the definition of \(R^*\) as an infimum over all classifiers. For the upper bound, as \(n \to \infty\), the 1-NN of any test point \(X\) converges in distribution to an independent draw from \(P(Y \mid X)\). Let \(\eta_c(x) = P(Y = c \mid X = x)\) and \(\eta^*(x) = \max_{c}\eta_c(x)\). The asymptotic 1-NN error is:
\[
  R^{\text{1NN}} = \mathbb{E}\left[1 - \sum_{c \in \mathcal{C}} \eta_c(X)^2\right]
\]
Using the Cauchy-Schwarz inequality, \(\sum_c \eta_c^2 \geq (\eta^*)^2 \geq 2\eta^* - 1\), which implies \(1 - \sum_c \eta_c^2 \leq 2(1 - \eta^*)\). Taking the expectation on both sides and recognizing that \(R^* = \mathbb{E}[1 - \eta^*(X)]\), we obtain the asymptotic bound:
\[
  R^* \leq R^{\text{1NN}} \leq 2R^*
\]

\textbf{Part (iii) Consistency.} 
By the consistency of the empirical 1-NN error under Stone's theorem, \(R_n^{\text{1NN}} \to R^{\text{1NN}}\) almost surely as \(n \to \infty\). Taking the limit inferior of the finite-sample bound from Part (i):
\[
  \liminf_{n \to \infty} \rho_n \geq \lim_{n \to \infty} \frac{n}{2(n-1)}R_n^{\text{1NN}} = \frac{1}{2}R^{\text{1NN}} \geq \frac{R^*}{2}
\]
where the final inequality follows from the left side of the Cover-Hart sandwich. Consequently, if a sequence of representations yields a vanishing leakage rate (\(\rho_n \to 0\)), it forces \(R^{\text{1NN}} = 0\), which in turn guarantees \(R^* = 0\). This proves that the MST leakage rate is a theoretically rigorous proxy for boundary separability. \hfill \(\blacksquare\)

\subsection{Proof of Theorem \ref{thm:mst_preserve}: MST Topology Preservation under Random Projection}
\label{app:mst_proof}

To prove Theorem 1, we first introduce two fundamental lemmas regarding the structure of Minimum Spanning Trees (MST). Let \(\mathcal{X} = \{x_1, \ldots, x_n\} \subset \mathbb{R}^d\) be a point cloud. The complete graph \(G = \bigl([n],\, w\bigr)\) carries edge weights \(w_{ij} = \|x_i - x_j\|_2\). Let \(T_n = \text{MST}(\mathcal{X})\) denote the Euclidean MST and \(E(T_n)\) denote its edge set.

\begin{lemma}[Cut Property]
\label{lem:cut}
Let \(G=(V,w)\) be a connected weighted graph with strictly distinct edge weights. For any proper subset \(S \subsetneq V\), the minimum-weight edge in the cut \((S,\, V \setminus S)\) belongs to \(\text{MST}(G)\).
\end{lemma}

\begin{proof}
Let \(e^* = (u, v)\), where \(u \in S\) and \(v \notin S\), be the minimum-weight edge crossing the cut. Suppose for contradiction that \(e^* \notin E(\text{MST}(G))\). Since \(\text{MST}(G)\) is a spanning tree, there is a unique path \(P\) from \(u\) to \(v\) in \(\text{MST}(G)\). Because \(u \in S\) and \(v \notin S\), path \(P\) contains at least one edge \(e'\) crossing the cut. By our assumption, \(w(e^*) < w(e')\). Replacing \(e'\) by \(e^*\) produces a spanning tree of strictly smaller total weight, which contradicts the minimality of \(\text{MST}(G)\). Thus, \(e^* \in E(\text{MST}(G))\).
\end{proof}

\begin{lemma}[1-NN Edge Inclusion]
\label{lem:1nn}
Let \(\text{NN}(i) = \arg\min_{j \neq i}\|x_i - x_j\|_2\) be the nearest neighbor of \(x_i\) in \(\mathcal{X}\). Then \(\bigl(i,\, \text{NN}(i)\bigr) \in E(T_n)\) for every \(i \in [n]\).
\end{lemma}

\begin{proof}
Apply Lemma~\ref{lem:cut} to the cut \(S = \{i\}\). The unique minimum-weight edge leaving node \(i\) is precisely \(\bigl(i, \text{NN}(i)\bigr)\). By the Cut Property, this edge must be included in \(E(T_n)\).
\end{proof}

\paragraph{Proof of Theorem \ref{thm:mst_preserve}}
The theorem consists of a deterministic part and a probabilistic part.

\textbf{Part (i) Deterministic part.} 
Because all pairwise distances in \(\mathcal{X}\) are distinct, Kruskal's algorithm produces the unique MST \(T_n\) by inserting edges in strictly increasing order of weight. The edge set \(E(T_n)\) is therefore completely determined by the total order \(<\) on the \(\tbinom{n}{2}\) pairwise distances. It suffices to show that the isometry \(f\) preserves this total order.

Take any two edges \(e_1 = (i,j)\) and \(e_2 = (k,l)\) with \(d_1 := \|x_i - x_j\|_2 < d_2 := \|x_k - x_l\|_2\). By the definition of the minimum distance ratio \(r\), we have \(d_2/d_1 \geq r\). Under the \((1 \pm \varepsilon)\)-isometry, the projected distances satisfy:
\[
  \|f(x_i) - f(x_j)\|_2 \leq (1+\varepsilon)d_1
\]
\[
  \|f(x_k) - f(x_l)\|_2 \geq (1-\varepsilon)d_2 \geq (1-\varepsilon)r d_1
\]
Hence, \(\|f(e_1)\|_2 < \|f(e_2)\|_2\) is guaranteed whenever \((1+\varepsilon) < (1-\varepsilon)r\), which simplifies to:
\[
  \varepsilon(1 + r) < r - 1 \iff \varepsilon < \frac{r-1}{r+1}
\]
Since this condition holds by assumption, the total order \(<\) of all edge weights is perfectly preserved. Consequently, Kruskal's algorithm processes the edges in the exact same sequence, yielding \(E\bigl(\text{MST}(f(\mathcal{X}))\bigr) = E(T_n)\).

\textbf{Part (ii) Probabilistic part.} 
Fix any pair \(i \neq j\) and let \(u = x_i - x_j\). The random projection is \(f(u) = \frac{1}{\sqrt{m}}\mathbf{A}u\). We can write:
\[
  \|f(u)\|_2^2 = \frac{\|u\|_2^2}{m}\sum_{k=1}^{m} Z_k^2, \qquad \text{where } Z_k = \frac{\mathbf{A}_k \cdot u}{\|u\|_2} \overset{\text{i.i.d.}}{\sim} \mathcal{N}(0,1)
\]
Here, \(\mathbf{A}_k\) is the \(k\)-th row of \(\mathbf{A}\). Let \(S_m = \frac{1}{m}\sum_{k=1}^m Z_k^2 \sim \frac{1}{m}\chi^2_m\). Using the Laurent-Massart chi-squared tail bound for \(\varepsilon \in (0,1)\):
\[
  \mathbb{P}\bigl[|S_m - 1| > \varepsilon\bigr] \leq 2\exp\left(-\frac{m\varepsilon^2}{8}\right)
\]
When \(|S_m - 1| \leq \varepsilon\), we have \(\|f(u)\|_2/\|u\|_2 = \sqrt{S_m} \in [\sqrt{1-\varepsilon},\,\sqrt{1+\varepsilon}] \subset [1-\varepsilon,\,1+\varepsilon]\). 
Applying a union bound over all \(\binom{n}{2} \leq \frac{n(n-1)}{2}\) pairs gives:
\[
  \mathbb{P}\bigl[\exists i \neq j: |S_m^{(ij)} - 1| > \varepsilon\bigr] \leq n(n-1)\exp\left(-\frac{m\varepsilon^2}{8}\right) \leq \delta
\]
Solving for \(m\) yields the requirement \(m \geq \frac{8}{\varepsilon^2}\ln\frac{n(n-1)}{\delta}\). Conditioned on this event (which occurs with probability at least \(1-\delta\)), \(f\) is a \((1\pm\varepsilon)\)-isometry. By Part (i), the MST topology is preserved. \hfill \(\blacksquare\)

\begin{proof}[Proof of Proposition~\ref{prop:variance_reduction}]
Since $\mathbf{A} \in \mathbb{R}^{Sm \times d}$ has i.i.d.\ $\mathcal{N}(0,1)$
entries and $u \in \mathbb{R}^d$ is fixed, the quantity
$\|\mathbf{A}u\|_2^2 / \|u\|_2^2 \sim \chi^2_{Sm}$.
Hence $\hat{D}_S = \frac{\|u\|_2^2}{Sm}\,\chi^2_{Sm}$, which gives
\[
    \mathbb{E}[\hat{D}_S] = \|u\|_2^2 \quad (\text{unbiased}), \qquad
    \operatorname{Var}[\hat{D}_S]
    = \frac{\|u\|_2^4}{(Sm)^2}\cdot 2Sm
    = \frac{2\,\|u\|_2^4}{Sm}.
\]
Applying the same argument to $\hat{D}_1 = \frac{\|u\|_2^2}{m}\chi^2_m$ yields
$\operatorname{Var}[\hat{D}_1] = \frac{2\|u\|_2^4}{m}$, so
$\operatorname{Var}[\hat{D}_S] = \frac{1}{S}\operatorname{Var}[\hat{D}_1]$.
The concentration bound follows directly from the Laurent--Massart chi-squared
tail bound applied to the $Sm$-dimensional projection
$\hat{D}_S \equiv \frac{\|u\|_2^2}{Sm}\chi^2_{Sm}$:
\[
    \mathbb{P}\!\left[|S_m - 1| > \varepsilon\right]
    \leq 2\exp\!\left(-\frac{Sm\,\varepsilon^2}{8}\right),
    \quad \text{where } S_m = \frac{\hat{D}_S}{\|u\|_2^2}. \qquad \square
\]
\end{proof}

\section{Social Impacts and Future Work}
\label{app:dis}

\subsection{Broader Impacts}
\label{sec:broader_impacts}

This work introduces a training-free, topology-driven transferability estimation (TE) framework for 3D medical vision foundation models. By enabling efficient and accurate model selection without exhaustive fine-tuning, our research presents several broader societal and industrial implications, encompassing both positive contributions and potential risks.

\paragraph{Positive Societal Impacts} 
First, our framework significantly lowers the computational barrier to deploying advanced 3D medical AI. By accelerating the model selection process compared to traditional dense metrics, it democratizes access to state-of-the-art medical foundation models for smaller research labs, under-resourced hospitals, and clinics that lack massive GPU clusters. Second, from an environmental perspective, eliminating the need for exhaustive fine-tuning across a growing zoo of foundation models drastically reduces energy consumption and the associated carbon footprint, contributing to the paradigm of ``Green AI.'' Finally, in clinical and industrial settings, our method accelerates the development pipeline of diagnostic tools, allowing practitioners to rapidly identify the most suitable models for emerging diseases or novel imaging modalities.

\paragraph{Potential Negative Impacts and Risks} 
Despite its efficiency, our method carries potential risks if misused. The primary concern is the \textit{propagation of pre-training biases}. If a foundation model encodes historical biases (e.g., underperforming on specific demographic groups or rare pathologies due to imbalanced pre-training data), our TE metric might still rank it highly if its overall feature topology aligns with the target labels. Consequently, an efficient selection tool could inadvertently accelerate the deployment of biased models without the scrutiny that typically accompanies exhaustive fine-tuning. Furthermore, there is a risk of \textit{automation bias}, where practitioners might over-rely on the TE score and bypass rigorous downstream clinical validation, especially in extreme out-of-distribution (OOD) or few-shot scenarios where the Minimum Spanning Tree (MST) construction might be less robust.

\paragraph{Mitigation Strategies} 
To mitigate these risks, we strongly advocate that our transferability estimation framework should be utilized as a \textit{preliminary filtering mechanism} rather than a definitive certification for clinical deployment. While it efficiently narrows down the candidate pool, the final selected model must still undergo rigorous, task-specific fine-tuning, fairness auditing, and exhaustive clinical validation across diverse patient demographics before real-world medical application. Future work should also explore integrating fairness-aware penalties into the topological alignment metrics to explicitly penalize models that exhibit demographic disparities.

\subsection{Future Work}
While this work focuses on transferability estimation for medical vision foundation model selection, the proposed topological metrics naturally suggest several promising extensions. First, the layer-wise GRTD and LBTC profiles can be repurposed to guide fine-tuning itself: layers exhibiting high local boundary alignment (high LBTC) may be safely frozen, while layers with poor global manifold correspondence (high GRTD) should be prioritized for parameter updates, enabling a topology-diagnosed layer-wise fine-tuning strategy that integrates readily with parameter-efficient methods such as LoRA and Adapter tuning. 
Also, beyond post-hoc model ranking, periodically computing topological metrics on a small labeled probe set during pre-training could serve as an online diagnostic signal, allowing practitioners to predict downstream transferability trajectories and identify the point of diminishing returns in continued pre-training. 

Additionally, the empirically observed phenomenon where randomly initialized decoders preserve transferable topological structure comparably to trained ones warrants deeper theoretical investigation. Specifically, it remains an open question whether it generalizes beyond U-Net-style skip-connection architectures to pure Transformer decoders (e.g., SAM's mask decoder) or diffusion model denoisers, and whether it can be formally unified with the reservoir computing literature, where randomly initialized recurrent networks are known to preserve high-dimensional topological structure of their inputs. Together, these directions position the topological transferability framework as a foundation not only for model selection, but for a broader geometry-aware paradigm of pre-training analysis and adaptation.


\end{document}